\documentclass{article}
\usepackage{fullpage,amsmath,amssymb, amsthm}
\usepackage[utf8]{inputenc} 
\usepackage[T1]{fontenc}    
\usepackage{hyperref}       
\usepackage{url}            
\usepackage{booktabs}       
\usepackage{amsfonts}       
\usepackage{nicefrac}       
\usepackage{microtype}      
\usepackage{multirow}
\usepackage{graphicx}
\usepackage{bm}
\usepackage{algorithm}
\usepackage[noend]{algpseudocode}
\usepackage{textcomp}

\usepackage[numbers]{natbib}

\newtheorem{thm}{Theorem}
\newtheorem{prop}[thm]{Proposition}

\newtheorem{lemma}[thm]{Lemma}
\newtheorem{defn}[thm]{Definition}

\newcommand{\R}{\mathbb{R}}

\newcommand{\bX}{\bm{X}}
\newcommand{\bY}{\bm{Y}}

\newcommand{\sE}{{\mathcal E}}
\newcommand{\sF}{{\mathcal F}}

\newcommand{\sH}{{\mathcal H}}

\newcommand{\sX}{{\mathcal X}}

\renewcommand{\a}{\alpha}
\renewcommand{\b}{\beta}

\newcommand{\g}{\gamma}
\newcommand{\eps}{\epsilon}

\newcommand{\sign}{\mathop{\mathrm{sign}}}
\newcommand{\argmin}{\operatornamewithlimits{arg\ min}}

\newcommand{\ind}[1]{{\bf 1}_{\{#1\}}}
\newcommand{\norm}[1]{\left\|#1\right\|}
\newcommand{\abs}[1]{\left\lvert #1 \right\rvert}
\newcommand{\set}[1]{\left\{#1\right\}}

\newcommand{\brac}[1]{\left[#1\right]}

\newcommand{\bcase}{\left\{ \begin{array}{ll} }
\newcommand{\ecase}{\end{array} \right. }
\newcommand{\bbP}{\mathbb{P}}
\newcommand{\bbE}{\mathbb{E}}
\newcommand{\ee}[2]{\bbE_{#1}\brac{#2}}
\newcommand{\e}[1]{\bbE\brac{#1}}

\newcommand{\sEhat}{\widehat{\sE}}
\newcommand{\fhat}{\widehat{f}}

\newcommand{\ellt}{\tilde{\ell}}
\newcommand{\Yt}{\tilde{Y}}

\newcommand{\Pt}{\tilde{P}}

\newcommand{\Pp}{P_+}
\newcommand{\Pm}{P_-}
\newcommand{\cpp}{\kappa^+}
\newcommand{\cpm}{\kappa^-}
\newcommand{\lp}{\gamma}
\newcommand{\blp}{\bm{\gamma}}
\newcommand{\lpp}{\gamma^+}
\newcommand{\lpm}{\gamma^-}
\newcommand{\Lp}{\Gamma}
\newcommand{\Lpp}{\Lp^+}
\newcommand{\Lpm}{\Lp^-}
\newcommand{\Lphat}{\widehat{\Lp}}
\newcommand{\Lpphat}{\widehat{\Lp}^+}
\newcommand{\Lpmhat}{\widehat{\Lp}^-}

\newcommand{\Lppbp}{\underline{\Lp}^+}
\newcommand{\Lpmbp}{\underline{\Lp}^-}
\newcommand{\Lpphatbp}{\widehat{\underline{\Lp}}^+}
\newcommand{\Lpmhatbp}{\widehat{\underline{\Lp}}^-}

\newcommand{\lphat}{\widehat{\lp}}
\newcommand{\lpphat}{\widehat{\lp}^+}
\newcommand{\lpmhat}{\widehat{\lp}^-}
\newcommand{\cp}{\kappa}
\newcommand{\cphat}{\widehat{\kappa}}
\newcommand{\sEt}{\tilde{\mathcal{E}}}
\newcommand{\sig}{\sigma}
\newcommand{\ubar}[1]{\underline{#1}}

\newcommand{\IIM}{{\bf (IIM)} }
\newcommand{\IBM}{{\bf (IBM)} }
\newcommand{\CIIM}{{\bf (CIIM)} }
\newcommand{\CIBM}{{\bf (CIBM)} }
\newcommand{\LP}{{\bf (LP)} }

\newcommand{\SR}{{\bf (SR)} }
\newcommand{\HM}{\textsc{HM}}
\newcommand{\AM}{\textsc{AM}}
\newcommand{\EPR}{\textsc{EPR}}

\title{Learning from Label Proportions: \\A Mutual Contamination Framework}

\author{Clayton Scott and Jianxin Zhang \\
Electrical Engineering and Computer Science\\
University of Michigan
}

\begin{document}

\maketitle

\begin{abstract}
Learning from label proportions (LLP) is a weakly supervised setting for classification in which unlabeled training instances are grouped into bags, and each bag is annotated with the proportion of each class occurring in that bag. Prior work on LLP has yet to establish a consistent learning procedure, nor does there exist a theoretically justified, general purpose training criterion. In this work we address these two issues by posing LLP in terms of mutual contamination models (MCMs), which have recently been applied successfully to study various other weak supervision settings. In the process, we establish several novel technical results for MCMs, including unbiased losses and generalization error bounds under non-iid sampling plans. We also point out the limitations of a common experimental setting for LLP, and propose a new one based on our MCM framework.
\end{abstract}

\section{Introduction}

Learning from label proportions (LLP) is a weak supervision setting for classification. In this problem, training data come in the form of bags. Each bag contains unlabeled instances and is annotated with the proportion of instances arising from each class. Various methods for LLP have been developed, including those based on support vector machines and related models \cite{rueping10, yu13icml, wang15, qi17, chen17npsvm, lai14cvpr, shi17}, Bayesian and graphical models \cite{kuck05uai, hernandez13, sun17, poyiadzi18, hernandez18}, deep learning \cite{li15alter, ardehaly17, dulac19arxiv, liu19neurips, tsai2020learning}, clustering \cite{chen09, stolpe11}, and random forests \cite{shi18}. In addition, LLP has found various applications including image and video analysis \citep{chen14,lai14cvpr}, high energy physics \cite{dery18}, vote prediction \cite{sun17}, remote sensing \cite{li15alter, ding2017}, medical image analysis \cite{bortsova18}, activity recognition \cite{poyiadzi18}, and reproductive medicine \cite{hernandez18}. 

Despite the emergence of LLP as a prominent weak learning paradigm, the theoretical underpinnings of LLP have been slow to develop. In particular, prior work has not established an algorithm for LLP that is consistent with respect to a classification performance measure. Furthermore, there does not even exist a general-purpose, theoretically grounded empirical objective for training LLP classifiers.

We propose a statistical framework for LLP based on mutual contamination models (MCMs), which have been used previously as models for classification with noisy labels and other weak supervision problems \cite{scott13colt, blanchard14, menon15icml, blanchard16ejs, katzsam19jmlr}. We use this framework to motivate a principled empirical objective for LLP, prove generalization error bounds associated to two bag generation models, and establish universal consistency with respect to the balanced error rate (BER). The MCM framework further motivates a novel experimental setting that overcomes a limitation of earlier experimental comparisons.

{\bf Related Work.}
\citet{quadrianto09jmlr} study an exponential family model for labels given features, and show that the model is characterized by a certain ``mean map" parameter that can be estimated in the LLP setting. They also provide Rademacher complexity bounds for the mean map and the associated log-posterior, but do not address a classification performance measure. \citet{patrini14nips} extend the work of \cite{quadrianto09jmlr} in several ways, including a generalization error bound on the risk of a classifier. This bound is expressed in terms of an empirical LLP risk, a ``bag Rademacher complexity," and a ``label proportion complexity." The authors state that when bags are pure (LPs close to 0 or 1), the last of these terms is small, while for impure bags, the second term is small and the first term increases. While this bound motivates their algorithms, it is not clear how such a bound would imply consistency. \citet{yu15tr} study the idea of minimizing the ``empirical proportion risk" (EPR), which seeks a classifier that best reproduces the observed LPs. They develop a PAC-style bound on the accuracy of the resulting classifier under the assumption that all bags are very pure. Our work is the first to develop generalization error analysis and universal consistency for a classification performance measure, and we do so under a broadly applicable statistical model on bags.

The literature on LLP has so far yielded two general purpose training objectives that are usable across a variety of learning models. The first of these, the aforementioned EPR, minimizes the average discrepancy between observed and predicted LPs, where discrepancy is often measured by absolute or squared error in the binary case \cite{yu15tr, tsai2020learning, dery18}, and cross-entropy in the multiclass case
\cite{tsai2020learning, dulac19arxiv, liu19neurips, bortsova18}. While \cite{yu15tr} has been cited as theoretical support for this objective, that paper assumes the bags are very pure, and even provides examples of EPR minimization failure when bags are not sufficiently pure. We offer our own counterexample in an appendix. The second is the combinatorial objective introduced by \cite{yu13icml} that incorporates the unknown labels as variables in the optimization, and jointly optimizes a conventional classification empirical risk together with a term (usually EPR) that encourages correctness of the imputed labels \cite{yu13icml, wang15, li15alter, qi17, chen17npsvm, shi17, shi18, lai14cvpr, dulac19arxiv}. To our knowledge there is also no statistical theory supporting this objective. In contrast, we propose a theoretically grounded, general purpose criterion for training LLP models.

Finally, we note that an earlier version of this work approached LLP using so-called ``label-flipping" or ``class-conditional" noise models, as opposed to MCMs \cite{scott19arxiv}. While that approach lead to the same algorithm described here, that setting is less natural for LLP, and the present version adds several more theoretical and experimental results.



{\bf Notation.} Let $\sX$ denote the feature space and $\{-1,1\}$ the label space. For convenience we often abbreviate $-1$ and $+1$ by ``-" and ``+", and write $\{\pm\} = \{-,+\}$. A binary classification loss function, or {\em loss} for short, is a function $\ell:\R \times \{-1,1\} \to \R$ (we allow losses to take negative values). For $\sig \in \{\pm\}$, denote $\ell_{\sig}(t):= \ell(t,\sig)$. A loss $\ell$ is {\em Lipschitz (continuous)} if there exists $L$ such that for every $\sig \in \{\pm\}$, and every $t, t' \in \R$, $|\ell_\sig(t)-\ell_\sig(t')| \le L|t - t'|$. The smallest such $L$ for which this property holds is denoted $|\ell|$. Additionally, we define $|\ell|_0:=\max(|\ell_+(0)|,|\ell_-(0)|)$.


A decision function is a measurable function $f:\sX \to \R$. The classifier induced by a decision function $f$ is the function $x \mapsto \sign(f(x))$. We will only consider classifiers induced by a decision function. In addition, we will often refer to a decision function as a classifier, in which case we mean the induced classifier. Let $\Pp$ and $\Pm$ be the class-conditional distributions of the feature vector $X$, and denote $P = (\Pm,\Pp)$. The performance measure considered in this work is the {\em balanced error rate} (BER) which, for a given loss $\ell$, and class conditional distributions $P = (\Pp,\Pm)$, is defined by $\sE_{P}^\ell(f):=\frac12 \bbE_{X \sim \Pp}[\ell_+(f(X))] + \frac12\bbE_{X \sim \Pm}[\ell_-(f(X))]$. 

For an integer $n$, denote $[n]:= \{1,2,\ldots,n\}$.
Given a sequence of numbers $(a_i)_{i\in [m]}$, denote the arithmetic and harmonic means by $\AM(a_i):=\frac1{m}\sum_{i \in [m]}a_i$ and $\HM(a_i):=(\frac1{m}\sum_{i \in [m]} a_i^{-1})^{-1}$. Finally, define the probability simplex $\Delta^N := \{w \in \R^N \, | \, w_i \ge 0 \, \forall i, \text{ and } \sum_i w_i =1\}$.

\section{Mutual Contamination Models}
\label{sec:unbiased}

In this section we define MCMs and present new technical results for learning from MCMs that motivate our study of LLP in the next section, and which may also be of independent interest. We will consider collections of instances $X_1, \ldots, X_m \sim \lp \Pp + (1-\lp) \Pm$, where $\gamma \in [0,1]$ and $m$ are fixed. Foreshadowing  LLP, we refer to such collections of instances as \emph{bags}. 

We adopt the following assumption on bag data generation, with two cases depending on within-bag dependencies. Suppose there are $L$ total bags with sizes $n_i$, $i\in[L]$, proportions $\lp_i \in [0,1]$, and elements $X_{ij}$, $i\in[L],j\in[n_i]$. We assume
\begingroup
\addtolength\leftmargini{-0.3in}
\begin{quote}
The distributions $\Pp$ and $\Pm$ are the same for all bags. $\lp_i$ and $m_i$ may vary from bag to bag. If $i \ne r$, then $X_{ij}$ and $X_{rs}$ are independent $\forall j,s$. Furthermore, for all $i$,
\begin{description}
\item[(IIM)] In the {\em independent instance model}, $X_{ij} \stackrel{iid}{\sim} \lp_i \Pp + (1-\lp_i) \Pm$;
\item[(IBM)] In the {\em independent bag model}, the marginal distribution of $X_{ij}$ is $\lp_i \Pp + (1-\lp_i) \Pm$.
\end{description}
\end{quote}
\endgroup
\IBM allows the instances within each bag to be dependent. Furthermore, any dependence structure, such as a covariance matrix, \emph{may change from bag to bag}. \IIM is a special case of \IBM that allows us to quantify the impact of bag size $n_i$ on generalization error.

\subsection{Mutual Contamination Models and Unbiased Losses}

Recall that $P$ denotes the pair $(\Pp, \Pm)$. Let $\cp = (\cpp,\cpm)$ be such that $\cpp + \cpm < 1$. A {\em mutual contamination model} is the pair $P^\cp := (\Pp^\cp,\Pm^\cp)$ where
\[
\Pp^\cp:=(1-\cpp)\Pp + \cpp \Pm \qquad \text{and} \qquad
\Pm^\cp:=(1-\cpm)\Pm + \cpm \Pp.
\]
$\Pp^\cp$ and $\Pm^\cp$ may be thought of as noisy or contaminated versions of $\Pp$ and $\Pm$, respectively, where the contamination arises from the other distribution. MCMs are common models for label noise \cite{scott13colt, menon15icml, blanchard16ejs}, where $\cp^\sig$ may be interpreted as the label noise rates $\bbP(Y=-\sig | \Yt = \sig)$, where $Y$ and $\Yt$ are the true and observed labels. 


Given $\ell$ and $\cp$ define the loss $\ell^\cp$ by
\[
\ell_\sig^\cp(t) := \frac{1-\cp^{-\sig}}{1 - \cpm - \cpp} \ell_\sig(t) - \frac{\cp^{-\sig}}{1 - \cpm - \cpp} \ell_{-\sig}(t), \qquad \sig \in \{\pm\}.
\]
This loss undoes the bias present in the mutual contamination model.

\begin{prop}
\label{prop:unbiased}
Consider any $P=(\Pp,\Pm)$, $\cp = (\cpp,\cpm)$ with $\cpp + \cpm < 1$, and loss $\ell$. For any $f$ such that all four of the quantities $\bbE_{X \sim P_{\pm}} \ell_{\pm}(f(X))$
exist and are finite, $\sE_P^\ell(f) = \sE_{P^\cp}^{\ell^\cp}(f)$.
\end{prop}
This result mirrors a similar result established by \citet{natarajan18jmlr} under a label-flipping model for label noise, which is the other prominent models for random label noise besides the MCM. The proof simply matches coefficients of $\bbE_{X \sim P_{\pm}} \ell_{\pm}(f(X))$ on either side of the desired identity. 

In an appendix we 
offer a sufficient condition for $\ell^\cp$ to be convex. We also 
show (as an aside) that Prop. \ref{prop:unbiased} enables a simple proof of a known result concerning symmetric losses, i.e., losses for which $\ell(t,1) + \ell(t,-1)$ is constant, such as the sigmoid loss. In particular, symmetric losses are immune to label noise under MCMs, meaning the original loss $\ell$ can be minimized on data drawn from the MCM and still optimize the clean BER \cite{menon15icml,rooyen15tr,charoenphakdee19icml}.

The significance of Prop. \ref{prop:unbiased} is that $\sE_P^\ell(f)$ is the quantity we want to minimize, while $ \sE_{P^\cp}^{\ell^\cp}(f)$ can be  estimated given data from an MCM. In particular, given bags $X_1^+, \ldots, X_{n^+}^+ \sim \Pp^\cp$ and $X_1^-, \ldots, X_{n^-}^- \sim \Pm^\cp$, 
Prop. \ref{prop:unbiased} motivates minimizing the estimate of BER given by
\[
\widehat{\sE}(f) :=
\frac1{2n^+} \sum_{j=1}^{n^+} \ell_+^\cp(f(X_j^+)) + \frac1{2n^-} \sum_{j=1}^{n^-} \ell_-^\cp(f(X_j^-)) = \frac12 \sum_{\sig \in \{\pm\}} \frac1{n^\sig} \sum_{j=1}^{n^\sig} \ell_\sig^\cp(f(X_j^\sig))
\]
over $f \in \sF$, where $\sF$ is some class of decision functions. We have
\begin{prop}
\label{prop:unbiasedest}
Under \IBM, for any $f$ such that the quantities $\bbE_{X \sim P_{\pm}} \ell_{\pm}(f(X))$ exist and are finite, $\bbE[\widehat{\sE}(f)] = \sE_P^\ell(f)$.
\end{prop}

\subsection{Learning from Multiple Mutual Contamination Models}

In the next section we view LLP in terms of a more general problem that we now define. Suppose we are given $N$ different MCMs. Each has the same true class-conditional distributions $\Pp$ and $\Pm$, but possibly different contamination proportions $\cp_i = (\cpp_i, \cpm_i)$, $i\in [N]$. Let $P^{\cp_i} = (\Pp^{\cp_i},\Pm^{\cp_i})$ denote the $i$th MCM, and assume $\cpp_i + \cpm_i < 1$.
Now suppose that for each $i \in [N]$, we observe 
\begin{align*}
X_{i1}^+,\ldots,X_{i{n_i^+}}^+ \sim \Pp^{\cp_i} &:= (1 - \cpp_i) \Pp + \cpp_i \Pm, \\
X_{i1}^-,\ldots,X_{i{n_i^-}}^- \sim \Pm^{\cp_i} &:= (1 - \cpm_i) \Pm + \cpm_i \Pp.
\end{align*}
The problem of {\em learning from multiple mutual contamination models} (LMMCM) is to use all of the above data to design a single classifier that minimizes the clean BER $\sE_P^\ell$.

A natural approach to this problem is to minimize the weighted empirical risk
\[
\widehat{\sE}_w(f):= \sum_{i=1}^N w_i \widehat{\sE}_i(f) \quad \text{where} \quad \widehat{\sE}_i(f) := \frac1{2n_i^+} \sum_{j=1}^{n_i^+} \ell_+^{\cp_i}(f(X_{ij}^+)) + \frac1{2n_i^-} \sum_{j=1}^{n_i^-} \ell_-^{\cp_i}(f(X_{ij}^-)), 
\]
where $w \in \Delta^N$. By Prop. \ref{prop:unbiased}, under \IBM each $\widehat{\sE}_i(f)$ is an unbiased estimate of $\sE_P^\ell(f)$, and therefore so is $\widehat{\sE}_w(f)$. This leads to the question of how best to set $w$.
Intuitively, MCMs $P^{\cp_i}$ with less corruption should receive larger weights. We confirm this intuition by choosing $w_i$ to optimize a generalization error bound (GEB). Our GEBs uses two weighted, multi-sample extensions of Rademacher complexity, corresponding to \IIM and \IBM, that we now introduce.


Let $S$ denote all the data $X_{ij}^\sig$ from $N$ MCMs as described above.
\begin{defn}
Let $\sF$ be a class of decision functions. Assume that $\sup_{f \in \sF} \sup_{x \in \sX} |f(x)| < \infty$. For any $c \in \R^N_{\ge 0}$, define 
\begin{equation}
\label{eqn:iimrad}
    \mathfrak{R}_c^I(\sF) := \bbE_{S} \bbE_{(\eps_{ij}^\sig)} \Bigg[ \sup_{f \in \sF} \sum_{i=1}^N c_i \sum_{\sig \in \{\pm\}} \frac1{2n_i^\sig} \sum_{j=1}^{n_i^\sig} \eps_{ij}^\sig f(X_{ij}^\sig) \Bigg],
\end{equation}
and 
\begin{equation}
\label{eqn:ibmrad}
    \mathfrak{R}_c^B(\sF) := \bbE_{S} \bbE_{{((\sig_i,X_i)\sim \widehat{P}^{\cp_i})}_{i\in [N]}} \bbE_{(\eps_{i})} \Bigg[ \sup_{f \in \sF} \sum_{i=1}^N \eps_i c_i f(X_i) \Bigg],
\end{equation}
where $\eps_{ij}^\sig, \eps_i \stackrel{iid}{\sim} \text{{\em unif}}(\{-1,1\})$ are Rademacher random variables and $\widehat{P}^{\cp_i}$ is the distribution that selects $\sig_i \sim \text{{\em unif}}(\{-1,1\})$, and then draws $X_i$ uniformly from $X_{i,1}^\sig, \ldots, X_{i,n_i^\sig}^\sig$.
\end{defn}
The inner two summations in \eqref{eqn:iimrad} reflect an adaptation of the usual Rademacher complexity to the BER, and the outer summation reflects the multiple MCMs.  Eqn. \eqref{eqn:ibmrad} may be seen as a modification of \eqref{eqn:iimrad} where the inner two sums are viewed as an empirical expectation that is pulled out of the supremum.
If $\sF$ satisfies the following, then $\mathfrak{R}_c^I(\sF)$ and $\mathfrak{R}_c^B(\sF)$ are bounded by tractable expressions.
\begin{description}
\item[(SR)] There exist constants $A$ and $B$ such that $\sup_{f \in \sF} \sup_{x \in \sX} |f(x)| \le A$, and for all $M$, $x_1, \ldots, x_M \in \sX$, and $a \in \R^M_{\ge 0}$,
\[
\ee{(\eps_i)}{\sup_{f \in \sF} \sum_{i =1}^M \eps_i a_i f(x_i)} \le B \sqrt{ \sum_{i=1}^M a_i^2}.
\]
\end{description}
As one example of an $\sF$ satisfying \SR, let $k$ be a symmetric positive definite (SPD) kernel, bounded\footnote{An SPD kernel $k$ is bounded by $K$ if $\sqrt{k(x,x)} \le K$ for all $x$. For example, the Gaussian kernel $k(x,x') = \exp(-\gamma \|x - x'\|^2)$ is bounded by $K=1$.} by $K$, and let $\sH$ be the associated reproducing kernel Hilbert space (RKHS). Let $\sF_{K,R}^k$ denote the ball of radius $R$, centered at 0, in $\sH$. 
As a second example, assume $\sX \subset \R^d$ and $\| \sX \|_2 := \sup_{x \in \sX} \| x\|_2 < \infty$, where $\| \cdot \|_2$ is the Euclidean norm. Let $\alpha, \beta \in \R_+^M$ and denote $\left[x \right]_+ = \max(0,x)$. Define the class of  two-layer neural networks with ReLU activation by
\[\sF^{\text{NN}}_{\alpha, \beta} = \{f(x) = v^T \left[ Ux \right]_{+} : v \in \mathbb{R}^h, U \in \mathbb{R}^{h \times d}, \abs{v_i} \leq \alpha_i, \norm{u_i}_2 \leq \beta_i, i = 1, 2, \dots, h \}. 
\]
\begin{prop}
\label{prop:sr}
$\sF_{K,R}^k$ satisfies \SR with $(A,B) = (RK,RK)$, and $\sF_{\a,\b}^{\text{NN}}$ satisfies \SR with $(A,B) = (\|\a\|_2 \|\b\|_2 \|\sX\|_2, 2 \langle \alpha, \beta \rangle \| \sX \|_2)$.
\end{prop}
We emphasize that other classes $\sF$ admit quantitative bounds on $\mathfrak{R}_c^I(\sF)$ and $\mathfrak{R}_c^B(\sF)$ that do not conform to \SR, and that can also be leveraged as we do below. We focus on \SR because the GEBs simplify considerably making it possible to derive closed form expressions for the optimal $w_i$. Below we write $\stackrel{\SR}{\le}$ to indicate an upper bound that holds provided \SR is true.


Our first main result establishes GEBs for LMMCM under both \IIM and \IBM.


\begin{thm}
\label{thm:multirad}
Let $S$ collect all the data $(X_{ij}^\sig)$ from $N$ MCMs with common base distributions $P_+, P_-$, and contamination proportions $\cp_i = (\cpp_i, \cpm_i)$ satisfying $\cpm_i + \cpp_i < 1$. Let $\sF$ be a class of decision functions such that $A = \sup_{f \in \sF} \sup_{x \in \sX} |f(x)| < \infty$, let $\ell$ a Lipschitz loss, $w \in \Delta^N$, and $\delta > 0$. Under \IIM, with probability $\ge 1 - \delta$ wrt the draw of $S$,
\begin{equation}
\label{eq:multirad}
\sup_{f \in \sF} \abs{
\widehat{\sE}_w(f) - \sE(f)} 
\leq 2 \mathfrak{R}_c^I(\sF) + C \sqrt{\sum_{i=1}^N  \frac{w_i^2}{\bar{n}_i (1 - \cpm_i - \cpp_i)^2} } \stackrel{\SR}{\le} D \sqrt{\sum_{i=1}^N  \frac{w_i^2}{\bar{n}_i (1 - \cpm_i - \cpp_i)^2}}
\end{equation}
where $\bar{n}_i:= \HM(n_i^-,n_i^+)$, $c_i = w_i |\ell|/(1-\cpm_i - \cpp_i)$, $C = (1 + A|\ell|) \sqrt{\log(2/\delta)}$, and $D=2 B|\ell| + C$. Under \IBM, the same statement holds after replacing $\mathfrak{R}_c^I(\sF) \to \mathfrak{R}_c^B(\sF)$ and $\bar{n}_i \to 1$.
\end{thm}
Several remarks are in order. Under \IIM, even in the special case $N=1$ without noise ($\cpm_1 = \cpp_1 = 0)$ the result appears new, and amounts to an adaptation of the standard Rademacher complexity bound to BER. The case $N=1$ {\em with} noise can be used to prove consistency (with $\bar{n}_1 \to \infty$) of a discrimination rule for a single $MCM$ given knowledge of, or consistent estimates of $\cpm_1, \cpp_1$. Previous results of this type have analyzed MCMs via label-flipping models which is less natural \cite{blanchard16ejs}.

Because the result holds for any $w \in \Delta^N$, as long as the $\cp_i$ are known a priori, we may set $w$ to optimize the rightmost expressions in \eqref{eq:multirad}. This leads to optimal weights $w_i \propto \bar{n}_i (1 - \cpm_i - \cpp_i)^2$ under \IBM (here and below, replace $\bar{n}_i$ by 1 for \IBM), which supports our claim that MCMs with more information (larger samples, less noise) should receive more weight. With this choice of weights, the summation in the bound reduces to $\frac1{N} \HM(1/\bar{n}_i (1 - \cpm_i - \cpp_i)^2)$. In contrast, with uniform weights $w_i = 1/N$ the summation equals $\frac1{N} \AM(1/\bar{n}_i (1 - \cpm_i - \cpp_i)^2)$. The harmonic mean is much less sensitive to the presence of outliers, i.e., very noisy MCMs, than the arithmetic. 


\section{Learning from Label Proportions}

In learning from label proportions with binary labels, the learner has access to $(b_1,\lphat_1), \ldots, (b_L, \lphat_L)$, where each $b_i$ is a bag of $n_i$ unlabeled instances, and each $\lphat_i \in [0,1]$ is the proportion of instances from class 1 in the bag. The goal is to learn an accurate classifier as measured by some performance measure, which in our case we take to be the BER. This choice is already a departure from prior work on LLP, which typically looks at misclassification rate (MCR). The BER is defined without reference to a distribution of the label $Y$, and is thus invariant to changes in this distribution. In other words, BER is immune to shifts in class prevalence, and hence to shifts in the distribution of label proportions.


We adopt the following data generation model for bags. Each bag has a {\em true label proportion} $\lp_i \in [0,1]$. For each $i$, let $(X_{ij},Y_{ij})$, $j \in [n_i]$, be random variables. The $i$th bag is formed from $(X_{ij})_{j \in [n_i]}$, and the \emph{observed} or \emph{empirical label proportion} is $\lphat_i = \frac1{n_i} \sum_j \frac{Y_{ij} + 1}2$. Let $\blp, \bY$, and $\bX$ be vectors collecting all of the values of $\lp_i, Y_{ij}$, and $X_{ij}$, respectively.
We assume

\begingroup
\addtolength\leftmargini{-0.3in}
\begin{quote}
The distributions $\Pp$ and $\Pm$ are the same for all bags. The $\lp_i$ may be random, and the sizes $n_i$ are nonrandom. Conditioned on $\blp$, if $i \ne r$, then $X_{ij}$ and $X_{rs}$ are independent $\forall j,s$. Furthermore, conditioned on $\blp$, for bag $i$
\begin{description}
\item[(CIIM)] In the {\em conditionally independent instance model}, $\frac{Y_{ij} + 1}2 \stackrel{iid}{\sim} \text{Bernoulli}(\lp_i)$ and conditioned on $Y_{i1}, \ldots, Y_{in_i}$, $X_{i1}, \ldots, X_{in_i}$ are independent with $X_{ij} \sim P_{Y_{ij}}$.
\item[(CIBM)] In the {\em conditionally independent bag model}, $\bbE[\lphat_i]=\lp_i$
and for each $j$, the distribution of  $X_{ij} | Y_{i1},\ldots,Y_{in_i}$ is $P_{Y_{ij}}$.
\end{description}
\end{quote}
\endgroup
Under \CIBM, conditioned on $\blp$, for bag $i$ the labels $Y_{i1}, \ldots, Y_{in_i}$ may be dependent, and given these labels the instances $X_{ij}$ may also be dependent. Furthermore, the dependence structure may change from bag to bag. This means that given its label, the distribution of an instance is still dependent on its bag, in contrast to prior work \cite{quadrianto09jmlr}. We also allow that the $\lp_i$ may be dependent, so that without conditioning on $\blp$, the bags themselves may be dependent.

As in the previous section, the significance of our model is that it provides for (conditionally) unbiased estimates of BER as we describe below. Indeed, if we view $\blp$ as fixed, \CIIM clearly implies  \IIM (in fact, the two independent instance models are equivalent). However, it is not the case that \CIBM implies \IBM -- the introduction of the latent labels allows for a more general independent bag model while still ensuring unbiased BER estimates. A weakening of \CIBM, namely
\begingroup
\addtolength\leftmargini{-0.3in}
\begin{quote}
\begin{description}
\item[(CIBM')] For each $j$, $\bbE[\frac{Y_{ij}+1}2]=\lp_i$ 
and the distribution of  $X_{ij} | Y_{i1},\ldots,Y_{in_i}$ is $P_{Y_{ij}}$
\end{description}
\end{quote}
\endgroup
does imply \IBM (still viewing $\blp$ as fixed), as we show in an appendix.

In this section we propose to reduce LLP to the setting of the previous section by pairing the bags, so that each pair of bags constitutes an MCM. 

\subsection{LLP when True Label Proportions are Known}
\label{sec:known}

We first consider the less realistic setting where the $\lp_i$ are deterministic and \emph{known}. In this situation we may reduce LLP to LMMCS by pairing bags. In particular, we re-index the bags and let $(b_i^-,\lp_i^-)$ and $(b_i^+,\lp_i^+)$ constitute the $i$th pair of bags, such that $\g_i^- < \g_i^+$. The bags may be paired in any way that depends on $\lp_1, \ldots, \lp_L$, subject to $\g_i^- < \g_i^+ \, \forall i$. We also assume the total number of bags is $L=2N$, so that the number of bag pairs is $N$. 

If we set $\cp_i = (\cpp_i, \cpm_i):=(1 - \lpp_i, \lpm_i)$, then we are in the setting of LMMCM described in the previous setting. Furthermore, $1 - \cpm_i - \cpp_i = \lpp_i - \lpm_i > 0$. Therefore we may apply all of the theory developed in the previous section without modification. Since $\blp$ is deterministic, \CIIM and {\bf (CIBM)'} imply \IIM and \IBM as discussed above, and we may simply apply Theorem \ref{thm:multirad} to obtain GEBs for LLP. Choosing weights $w_i$ to minimize the \SR form yields final bounds proportional to the square root of $\frac1{N}\HM(1/(\bar{n}_i (\lpp_i - \lpm_i)^2)) = (\sum_i \bar{n}_i (\lpp_i - \lpm_i)^2)^{-1}$ (under {\bf (CIBM') replace $\bar{n}_i \to 1$}). In the LLP setting, we may further optimize this bound by optimizing the pairing of bags. This leads to an integer program known as the weighted matching problem for which exact and approximate algorithms are known. See appendices for details. 

If $\blp$ is random, and the $\lp_i$ are distinct (which occurs w. p. 1, e.g., if $\blp$ is jointly continuous), Theorem \ref{thm:multirad} still holds conditioned on $\blp$, and therefore unconditionally by the law of total expectation. 

Although the $\lp_i$ are typically unknown in practice, the above discussion still yields a useful algorithm: simply ``plug in'' $\lphat_i$ for $\lp_i$ and proceed to minimize $\sEhat_w(f)$ (with optimally paired bags and optimized weights) over $\sF$. 
A description of the learning procedure, which we use in our experiments, is presented in Algorithm \ref{alg:plugin}.

\begin{algorithm}[H]
\caption{Plug-in approach to LLP via LMMCM (outline)} \label{alg:plugin}
\begin{algorithmic}[1]
\State \textbf{Input:} $(b_1,\lphat_1), \ldots, (b_{2N},\lphat_{2N})$, model class $\sF$, loss $\ell$, tuning parameters 
 \Procedure{LLP-LMMCM}{}
    \State{Solve weighted matching problem to find pairings maximizing $\sum_i (\lpphat_i - \lpmhat_i)^2$ (see supp.)}
    \State{Set $\cp_i = (1 - \lpphat_i, \lpmhat_i)$ and optimal weights $w_i \propto (\lpphat_i - \lpmhat_i)^2$}
    \State{Minimize $\sEhat_w(f)$ over $\sF$, perhaps with regularization}
 \EndProcedure
 \end{algorithmic}
\end{algorithm}

\subsection{Consistent Learning from Label Proportions}

When the true label proportions are not known, as is usually the case in practice, it is difficult to establish consistency of the plug-in approach without restrictive assumptions. This is because the $\lphat_i$ are random, and so there is always some nonnegligible probability that in each pair, the bag with larger $\lp_i$ will be misidentified. This problem is especially pronounced for very small bag sizes. For example, if two bags with $\lp_1 = .45$ and $\lp_2 = .55$ are paired, and the bag sizes are 8 with independent labels, the probability that $\lphat_2 < \lphat_1$ is .26. One approach to overcoming this issue is to have the bag sizes $n_i^{\sig}$ tend to $\infty$ asymptotically, in which case $\lphat_i \stackrel{a.s.}{\to} \gamma_i$. This is a less interesting setting, however, because the learner can discard all but one pair of bags and still achieve consistency using existing techniques for learning in MCMs \cite{blanchard16ejs}. Furthermore, the bag size is often fixed in applications.

We propose an approach based on merging the original ``small bags" to form ``big bags," and then applying the approach of Section \ref{sec:known}. For convenience assume all original (small) bags have the same size $n_i = n$ moving forward. Let $K$ be an integer and assume $N$ is a multiple of $K$ for convenience, $N = MK$. As before, let $(b_i,\lphat_i)$, $i \in [2N]$, be the original, unpaired bags of size $n$. We refer to a {\em K-merging scheme} as any procedure that takes the original unpaired bags of size $n$ and combines them, using knowledge of the $\lphat_i$, to form paired bags of size $nK$. Let the paired bags be denoted $(B_i^+, \Lpphat_i)$ and $(B_i^-, \Lpmhat_i)$, $i \in [M]$. Let $I_i^\sig$ denote the original indices of the small bags comprising $B_i^\sig$, so that $B_i^\sig = \cup_{j \in I_i^+} b_i$ and $\Lphat^\sig_i = \frac1{K} \sum_{j \in I_i^\sig} \lphat_j^\sig$.

We offer two examples of $K$-merging schemes. The first, called the {\em blockwise-pairwise (BP) scheme}, simply takes the original small bags in their given order. The $i$th block of 2K consecutive small bags are used to form the $i$th pair of big bags. This is done by considering consecutive, nonoverlapping pairs of small bags and assigning the small bag with larger $\lphat_i$ to $B_i^+$. Using notation, we define
$
I_i^+ = \{j \in [2K(i-1)+1:2Ki] \, | \, \text{$j$ is odd and } \lphat_{j} \ge \lphat_{j+1} \text{ or $j$ is even and } \lphat_{j} \ge \lphat_{j-1}\}
$
and $I_i^- = [2K(i-1)+1:2Ki] \backslash I_i^+$ (ties may be broken arbitrarily). The {\em blockwise-max (BM) scheme} is like BP, except that for each block of $2K$ small bags, the $K$ small bags with largest $\lphat_j$ are assigned to the positive bag. One can imagine more elaborate schemes that are not blockwise. We say that scheme 1 {\em dominates} scheme 2 if, with probability 1, for every $i$, $\Lpphat_i - \Lpmhat_i$ for scheme 1 is at least as large as it is for scheme 2. For example, BM dominates BP.

Next, we form the modified weighted empirical risk. For each $i \in [M]$ and $\sig \in \{\pm\}$, let $(X_{ij}^\sig)$, $j \in [nK]$, denote the elements of $B_i^\sig$, and $(Y_{ij}^\sig)$ the associated labels. Also set $\cphat_i = (1 - \Lpphat_i, \Lpmhat_i)$. Let $w \in \Delta^M$ such that $w_i \propto (\Lpphat_i - \Lpmhat_i)^2$, and define 
\[
\sEt(f):= \sum_{i=1}^M w_i \sEt_i(f)
\qquad \text{where} \qquad \sEt_i(f) :=
\left[ \frac1{2n} \sum_{\sig \in \{\pm\}}\sum_{j=1}^{nK} \ell_\sig^{\cphat_i}(f(X_{ij}^\sig)) \right].
\]
In the proof of Thm. \ref{thm:llpempbnd}, we show that under \CIBM, with high probability, $\sEt_i(f)$ is an unbiased estimate for $\sE_{P}^{\ell}(f)$ when conditioned on $\blp$ and $\bY$.


To state our main result we adopt the following assumption on the distribution of label proportions.  
\begin{description}
\item[\LP] There exist $\Delta, \tau > 0$ such that the sequence of random variables $Z_j = \ind{|\lp_{j} - \lp_{j+1}| < \Delta}$ satisfies the following. For every $J \subseteq [2N-1]$, 
$
\bbP(\prod_{j \in J} Z_j = 1) \le \tau^{|J|}.
$
\end{description}
This condition is satisfied if the $\lp_i$ are iid draws from any non-constant distribution. However, it also allows for the $\lp_i$ to be correlated. As one example, let $(w_j)$ be iid random variables with support $\supseteq [-1,1]$. \LP is satisfied if $\lp_{j+1} = \lp_j + \ubar{w}_j$, where $\ubar{w}_j$ is the truncation of $w_j$ to $[-\lp_j, 1-\lp_j]$. 
The point of \LP is that it offers a dependence setting where a one-sided version of Hoeffding's inequality holds, which allows us to conclude that with high probability, for all odd $j \in [2N]$, $|\lp_j - \lp_{j+1}| \ge \Delta$ for approximately $N(1 - \tau)$ of the original pairs of small bags \cite{panconesi97}.

We now state our main result. Define $\Lpp_i = \bbE_{\bY|\blp}[\Lpphat_i]$ and $\Lpm_i = \bbE_{\bY|\blp}[\Lpmhat_i]$.
\begin{thm}
\label{thm:llpempbnd}
Let \LP hold. Let $\eps_0 \in (0,\Delta(1-\tau))$. Let $\sF$ satisfy $\sup_{x \in \sX, f \in \sF} |f(x)| \le A < \infty$ and let $\ell$ be a Lipschitz loss. Let $\eps \in (0,\frac{\Delta(1-\tau)-\eps_0}{1+\Delta}]$ and $\delta \in (0,1]$. For the BP merging scheme, under \CIIM, with probability at least $1 - \delta - 2\frac{N}{K} e^{-2K \epsilon^2}$ with respect to the draw of $\blp, \bY, \bX$,
\[
\Lpphat_i - \Lpmhat_i \ge \Lpp_i - \Lpm_i - \eps \ge \eps_0
\]
and
\begin{equation}
\label{eqn:llpepmbnd}
\sup_{f \in \sF} \abs{
\sEt(f) - \sE(f)} \le 2 \mathfrak{R}_c^{I}(\sF) + C \sqrt{\frac{\HM((\Lpp_i - \Lpm_i - \eps)^{-2})}{2(N/K)n}} \stackrel{\SR}{\le} D \sqrt{\frac{\HM((\Lpp_i - \Lpm_i - \eps)^{-2})}{2(N/K)n}},
\end{equation}
where $c_i = w_i |\ell|/(\Lpp_i - \Lpm_i - \eps)$, $C = (1 + A|\ell|) \sqrt{\log(2/\delta)}$, and $D=2 B|\ell| + C$. Under {\bf (CIBM)}, the same bounds hold with the same probability if we substitute $\mathfrak{R}_{c}^I(\sF) \to \mathfrak{R}_{c}^B(\sF)$ and $n \to 1$. 
\end{thm}
This result states that BP achieves essentially the same bound (modulo $\eps$) as if we applied LMMCM to the big bags with \emph{known} $\Lpp_i, \Lpm_i$. We also note that there is no restriction on bag size $n$. A corollary of this result also applies to any scheme that dominates BP, as we explain in an appendix. 

Theorem \ref{thm:llpempbnd} implies a consistent learning algorithm for LLP under both \CIIM and \CIBM, using any merging scheme that dominates BP. To achieve consistency the bound should tend to zero while the confidence tends to 1, as $N \to \infty$. Even with $n$ fixed, this is true provided $K \to \infty$ and $N/K \to \infty$ as $N \to \infty$, such that $N = O(K^\beta)$ for some $\beta > 0$. Beyond that, standard arguments may be applied to arrive at a formal consistency result. In an appendix we state such a result for completeness. Here the consistency is \emph{universal} in that it makes no assumptions on $\Pm$ or $\Pp$.


\section{Experiments}

The vast majority of LLP methodology papers simulate data for LLP by taking a classification data set, randomly shuffling the data, and sectioning off the data into bags of a certain size. This implies that the expected label proportions for all bags are the same, and as bag size increases, all label proportions converge to the class prior probabilities. The case where all LPs are the same is precisely the setting where LLP becomes intractable, and hence these papers report decreasing performance with increasing bag size. 

We propose an alternate sampling scheme inspired by our MCM framework. Each experiment is based on a classification data set, a distribution of LPs, and the bag size $n$. For each dataset, the total number of training instances $T$ is fixed, so that the number of bags is $T/n$. We consider the Adult ($T=8192$) and MAGIC Gamma Ray Telescope ($T=6144$) datasets (both available from the UCI repository\footnote{http://archive.ics.uci.edu/ml}), LPs that are iid uniform on $[0,\frac12]$ and on $[\frac12, 1]$, and bag sizes $n \in \{8, 32, 128, 512\}$. The total number of experimental settings is thus $2 \times 2 \times 4 = 16$. The numerical features in both datasets are standardized to have 0 mean and unit variance, the categorical features are one-hot encoded.

We implement a method based on our general approach (see Algorithm \ref{alg:plugin}) by taking $\ell$ to be the logistic loss, $\sF$ to be the RKHS associated to a Gaussian kernel $k$, and selecting $f \in \sF$ by minimizing $\sEhat_w(f) + \lambda \| f \|_{\sF}^2$. By the representer theorem \cite{scholkopf01representer}, the minimizer of this objective has the form $f(x) = \sum_{i} \a_i k(x,x_i)$ where $\a_i \in \R$ and $x_i$ ranges over all training instances. Our Python implementation uses SciPy's L-BFGS routine to find the optimal $\a_i$. The kernel parameter is computed by $\frac{1}{d * Var(X)}$ where $d$ is the number of features and $Var(X)$ is the variance of the data matrix, and the parameter $\lambda \in \{1, 10^{-1}, 10^{-2}, \ldots, 10^{-5}\}$ is chosen by 5-fold cross validation. We tried the EPR as a criterion for model selection but found our own criterion to be better. For each dataset, our implementation runs all 8 settings in roughly 50 minutes using 48 cores.

We compare against InvCal \cite{rueping10} and alter-$\propto$SVM \cite{yu13icml}, the two most common reference methods in LLP, using Matlab implementations provided by the authors of \cite{yu13icml}. Those methods are designed to optimize accuracy, whereas ours is designed to optimize BER. For a fair comparison, for each method we shift the decision function's threshold to generate an ROC curve and evaluate the area under the curve (AUC) using all data that was not used for training. For each experimental setting, the reported AUC and standard deviation reflect the average results over 5 randomized trials. Additional experimental details are found in an appendix.

The results are reported in Table \ref{tab:auc}. Bold numbers indicate that a method's mean AUC was the largest for that experimental setting. We see that for the smallest bag size, the methods all perform comparably, while for larger bag sizes, LMMCM exhibits far less degradation in performance. Using the
Wilcoxon signed-rank test, we find that LMMCM outperforms InvCal with p-value < 0.005.


\begin{table*}[h] \caption{AUC. Column header indicates bag size. \label{tab:auc}}
\resizebox{1.0\textwidth}{!}{
  \begin{tabular}{c c c c c c}
    \hline 
    Data set, LP dist & Method & 8 & 32 & 128 & 512 \\
    \hline
    \multirow{3}{*}{Adult, $\left[0, \frac{1}{2} \right]$} & InvCal & 
    0.8720 $\pm$ 0.0035 &
    0.8672 $\pm$ 0.0067 &
    0.8537 $\pm$ 0.0101 &
    0.7256 $\pm$ 0.0159
    \\
    & alter-$\propto$SVM& 
    0.8586 $\pm$ 0.0185 &
    0.7394 $\pm$ 0.0686 &
    0.7260 $\pm$ 0.0953 &
    0.6876 $\pm$ 0.1219 
    \\
    & LMMCM & 
    \textbf{0.8728 $\pm$ 0.0019} &
    \textbf{0.8693 $\pm$ 0.0047} &
    \textbf{0.8669 $\pm$ 0.0041} &
    \textbf{0.8674 $\pm$ 0.0040} 
    \\
    
    \hline
    \multirow{3}{*}{Adult, $\left[\frac{1}{2}, 1 \right]$} & InvCal & 
    \textbf{0.8680 $\pm$ 0.0021} &
    0.8598 $\pm$ 0.0073 &
    0.8284 $\pm$ 0.0093 &
    0.7480 $\pm$ 0.0500
    \\
    & alter-$\propto$SVM& 
    0.8587 $\pm$ 0.0097 &
    0.7429 $\pm$ 0.1473 &
    0.8204 $\pm$ 0.0318 &
    0.7602 $\pm$ 0.1215
    \\
    & LMMCM & 
    0.8584 $\pm$ 0.0164 &
    \textbf{0.8644 $\pm$ 0.0052} &
    \textbf{0.8601 $\pm$ 0.0045} &
    \textbf{0.8500 $\pm$ 0.0186} 
    \\
    
    \hline
    \multirow{3}{*}{MAGIC, $\left[0, \frac{1}{2} \right]$} & InvCal & 
    \textbf{0.8918 $\pm$ 0.0076} &
    0.8574 $\pm$ 0.0079 &
    0.8295 $\pm$ 0.0139 &
    0.8133 $\pm$ 0.0109
    \\
    & alter-$\propto$SVM& 
    0.8701 $\pm$ 0.0026 &
    0.7704 $\pm$ 0.0818 &
    0.7753 $\pm$ 0.0207 &
    0.6851 $\pm$ 0.1580 
    \\
    & LMMCM & 
    0.8909 $\pm$ 0.0077 &
    \textbf{0.8799 $\pm$ 0.0113} &
    \textbf{0.8753 $\pm$ 0.0157} &
    \textbf{0.8734 $\pm$ 0.0092} 
    \\
    
    \hline
    \multirow{3}{*}{MAGIC, $\left[\frac{1}{2}, 1 \right]$} & InvCal & 
    \textbf{0.8936 $\pm$ 0.0066} &
    0.8612 $\pm$ 0.0056 &
    0.8180 $\pm$ 0.0092 &
    0.8215 $\pm$ 0.0136
    \\
    & alter-$\propto$SVM& 
    0.8689 $\pm$ 0.0135 &
    0.8219 $\pm$ 0.0218 &
    0.8179 $\pm$ 0.0487 &
    0.7949 $\pm$ 0.0478 
    \\
    & LMMCM & 
    0.8911 $\pm$ 0.0083 &
    \textbf{0.8790 $\pm$ 0.0091} &
    \textbf{0.8684 $\pm$ 0.0046} &
    \textbf{0.8567 $\pm$ 0.0292} 
    \\
  \end{tabular}
}
\end{table*}

We performed an additional set of experiments where the number of bags $N$ remains fixed. For Adult dataset, the total number of bags is 16, and for MAGIC, it is 12. For each method, we generate an ROC curve and evaluate the area under the curve (AUC) using the test data. The average AUCs and the standard deviations over 5 random trials are reported in Table \ref{tab:fixedN}. Bold numbers indicate that a method's mean AUC was the largest for that experimental setting. We observe that LMMCM exhibits excellent performance in this setting as well. 

\begin{table*}[h] \caption{AUC. Column header indicates bag size. \label{tab:fixedN}}
\resizebox{1.0\textwidth}{!}{
  \begin{tabular}{c c c c c c}
    \hline 
    Data set, LP dist & Method & 8 & 32 & 128 & 512 \\
    \hline
    \multirow{3}{*}{Adult, $\left[0, \frac{1}{2} \right]$} & InvCal & 
    0.6427 $\pm$  0.0922&
    0.6545 $\pm$  0.0643&
    0.6518 $\pm$  0.0139&
    0.7230 $\pm$  0.0253
    \\
    & alter-$\propto$SVM& 
    0.6525 $\pm$  0.0817 &
    0.5959 $\pm$  0.1145 &
    0.6199 $\pm$  0.1267 &
    0.6419 $\pm$  0.0997
    \\
    & LMMCM & 
    \textbf{ 0.7299 $\pm$ 0.0796} &
    \textbf{ 0.7765 $\pm$ 0.0590} &
    \textbf{ 0.8329 $\pm$ 0.0166} &
    \textbf{ 0.8456 $\pm$ 0.0213} 
    \\
    
    \hline
    \multirow{3}{*}{Adult, $\left[\frac{1}{2}, 1 \right]$} & InvCal & 
    0.5973 $\pm$ 0.0740 &
    0.6634 $\pm$ 0.0864 &
    0.6408 $\pm$ 0.0216 &
    0.7218 $\pm$ 0.0170
    \\
    & alter-$\propto$SVM& 
    0.6035 $\pm$ 0.1626 &
    \textbf{ 0.7774 $\pm$ 0.0443} &
    0.5863 $\pm$ 0.2775 &
    0.7106 $\pm$ 0.2193
    \\
    & LMMCM & 
    \textbf{ 0.7228 $\pm$ 0.1048} &
    0.7674 $\pm$ 0.0586 &
    \textbf{ 0.8428 $\pm$ 0.0101} &
    \textbf{ 0.8588 $\pm$ 0.0091} 
    \\
    
    \hline
    \multirow{3}{*}{MAGIC, $\left[0, \frac{1}{2} \right]$} & InvCal & 
    \textbf{ 0.7381 $\pm$ 0.0439} &
     0.7828 $\pm$  0.0212 &
     0.7936 $\pm$  0.0371 &
     0.8196 $\pm$  0.0231
    \\
    & alter-$\propto$SVM& 
     0.5997 $\pm$  0.1163 &
     0.5376 $\pm$  0.1671 &
     0.6859 $\pm$  0.0371 &
     0.7193 $\pm$  0.1278
    \\
    & LMMCM & 
     0.7180 $\pm$  0.0450 &
    \textbf{ 0.7852 $\pm$ 0.7828} &
    \textbf{ 0.8140 $\pm$ 0.0463} &
    \textbf{ 0.8630 $\pm$ 0.0275} 
    \\
    
    \hline
    \multirow{3}{*}{MAGIC, $\left[\frac{1}{2}, 1 \right]$} & InvCal & 
    0.6741 $\pm$ 0.0673 &
    0.7405 $\pm$ 0.0433 &
    0.7876 $\pm$ 0.0249 &
    0.8135 $\pm$ 0.0132
    \\
    & alter-$\propto$SVM& 
    0.6589 $\pm$  0.1029 &
    0.6330 $\pm$  0.1254 &
    0.6790 $\pm$  0.1072 &
    0.7965 $\pm$  0.0708
    \\
    & LMMCM & 
    \textbf{ 0.6807 $\pm$ 0.0779} &
    \textbf{ 0.7639 $\pm$ 0.0335} &
    \textbf{ 0.7905 $\pm$ 0.0258} &
    \textbf{ 0.8491 $\pm$ 0.0245} 
    \\
  \end{tabular}
}
\end{table*}

\section{Conclusion}

We have introduced a principled framework for LLP based on MCMs. We have developed several novel results for MCMs, and used them to develop a statistically consistent procedure and an effective practical algorithm for LLP. The most natural direction for future work is to extend to multiclass. 


\appendix

\section{Failure Case for Empirical Proportion Risk Minimization}

We offer a simple example where minimizing the empirical proportion risk leads to suboptimal performance. Let $\Pm$ be uniform on $[0,1]$, with density $p_-(x) = \ind{x \in [0,1]}$, and let $\Pp$ have the triangular density function $p_+(x) = 2x \ind{x \in [0,1]}$. Suppose there is a single bag, and that the label proportion is $\lp = \frac12$. Also suppose $\sF$ consists of threshold classifiers $f_t(x) = \sign(x - t)$, $t \in [0,1]$. This class contains the optimal BER classifier (define wrt 0-1 loss) corresponding to $t^* = \frac12$. Now suppose we are in the infinite bag-size limit (which only makes the problem easier), so that the observed label proportion $\lphat$ is simply $\lp = \frac12$. Then we seek the threshold $t'$ that minimizes
\[
\EPR(t) := \left|\bbP(f_t(X)=1) - \frac12\right|^p.
\]
For any $p > 0$, $t'$ is the median of the marginal distribution of $X$, $\frac12 \Pm + \frac12 \Pp$, which equals $(\sqrt{5}-1)/2 \approx  0.62 \ne t^*$. Thus, minimizing EPR does not yield an optimal classifier for BER or for misclassification rate, which agrees with BER in this setting where the two classes are equally likely.

Now suppose there are $N$ bags, with label proportions $\lp_1, \ldots, \lp_N$ drawn iid from a distribution whose (population) mean and median are $\frac12$, such as the uniform distribution on $[0,1]$. The optimal BER classifier remains the same, with threshold $t^* = \frac12$. The optimal classifier wrt misclassification rate is also the same, assuming we view $\bbE[\lp_i]= \frac12$ as the class prior. In the infinite bag-size limit, EPR would seek the threshold $t'$ that minimizes
\[
\EPR_N(t) := \frac1{N} \sum_{i=1}^N \left|\bbP(f_t(X)=1) - \lp_i \right|^p.
\]
For $p=1$, EPR minimization selects $t'$ such that $\bbP(f_{t'}(X)=1)$ is the empirical median of $\lp_1, \ldots, \lp_N$, which will be near $\frac12$, which means $t'$ will be near $0.62$. For $p=2$, EPR minimization selects $t'$ such that $\bbP(f_{t'}(X)=1)$ is the empirical mean of $\lp_1, \ldots, \lp_N$, which will again be near $\frac12$, which again means $t'$ will be near $0.62$. 

More generally, based on the above example, EPR seems likely to fail whenever $\Pp$ and $\Pm$ are not sufficiently ``symmetric."

\section{Proofs of Results From Main Document}

This section contains the proofs.

\subsection{Proof of Proposition \ref{prop:unbiased}}

Consider the loss function $\ellt$ given by
\begin{align*}
\ellt_+(t) &= A \ell_+(t) - B \ell_-(t), \\   
\ellt_-(t) &= C \ell_-(t) - D \ell_+(t). 
\end{align*}
Equating $\sE_{P^\cp}^{\ellt}(f)$ to $\sE_P^\ell(f)$ yields four equations in the four unknowns A, B, C, and D, corresponding to the coefficients of $\bbE_{X \sim P_{\pm}} \ell_{\pm}(f(X))$. The unique solution to this system is $\ellt = \ell^\cp$.

\subsection{Proof of Proposition \ref{prop:sr}}

We begin with $\sF^k_{R,K}$. For any $R > 0$, $f\in \sF^k_{R,K}$, and $x \in \sX$,
$$
|f(x)| = |\langle f,k(\cdot,x)\rangle| \le \|f\|_{{\cal H}} \|k(\cdot,x)\|_{{\cal H}} = RK.
$$
by the reproducing property and Cauchy-Schwarz. Thus $A=RK$.

For the second part, the expectation may be bounded by a modification of the standard bound of Rademacher complexity for kernel classes. Thus,
\begin{align}
\bbE_{(\eps_{i})} \Bigg[ \sup_{f \in \sF^k_{R,K}} \sum_{i} a_i \eps_i f(x_i) \Bigg]
&= \ee{(\eps_i)}{\sup_{f \in \sF^k_{R,K}} \sum_{i} a_i \eps_i \langle f, k(\cdot, x_i) \label{eqn:repr} \rangle} \\
&= \ee{(\eps_i)}{\sup_{f \in \sF^k_{R,K}} \left\langle f, \sum_{i} a_i \eps_i k(\cdot, x_i) \right\rangle} \nonumber \\
&= \ee{(\eps_i)}{\left\langle R \frac{\sum_{i} a_i \eps_i k(\cdot, x_i)}{\| \sum_{i} a_i \eps_i k(\cdot, x_i) \|}, \sum_{i} a_i \eps_i k(\cdot, x_i) \right\rangle} \label{eqn:cs} \\
&= R \ee{(\eps_i)}{\sqrt{\Bigg\| \sum_{i} a_i \eps_i k(\cdot, x_i) \Bigg\|^2}} \nonumber \\
&\le R \sqrt{\ee{(\eps_i)}{\Bigg\| \sum_{i} a_i \eps_i k(\cdot, x_i) \Bigg\|^2}} \label{eqn:jensen} \\
&= R \sqrt{\sum_{i} a_i^2 \| k(\cdot,x_i) \|^2} \label{eqn:rad} \\
&\le RK \sqrt{\sum_{i=1}^M a_i^2} \label{eqn:repr2},
\end{align}
where \eqref{eqn:repr} uses the reproducing property, \eqref{eqn:cs} is the condition for equality in Cauchy-Schwarz, \eqref{eqn:jensen} is Jensen's inequality, \eqref{eqn:rad} follows from independence of the Rademacher random variables, and \eqref{eqn:repr2} follows from the reproducing property and the bound on the kernel.

Next, consider $\sF^\text{NN}_{\a,\b}$. For the first part we have for any $f \in \sF^\text{NN}_{\a,\b}$ and $x \in \sX$,
\begin{align*}
|f(x)| &= |\langle v, [Ux]_+\rangle| \\
&\le \| v \| \| [Ux]_+ \| \\
&\le \| \a \| \| [Ux]_+ \| \\
&\le \| \a \| \| Ux \| \\
&= \| \a \| \sqrt{\sum_j |\langle u_j, x \rangle|^2} \\
&\le \| \a \| \sqrt{\sum_j \| u_j \|^2 \| x \|^2} \\
&\le \| \sX \| \| \a \| \sqrt{\sum_j \| u_j \|^2} \\
&\le \| \sX \| \| \a \| \| \b_j\|.
\end{align*}

For the second part, observe
\begin{align}
     \mathbb{E}_{(\epsilon_k)}\left[\sup_{f \in \mathcal{F}}  \sum_{k=1}^{M} \epsilon_k a_k f(x_k)\right]
    & = \mathbb{E}_{(\epsilon_k)}\left[\sup_{f \in \mathcal{F}}  \sum_{k=1}^{M} \epsilon_k a_k \sum_{j=1}^h v_j \left[\langle {u}_j, {x}_k \rangle\right]_+\right] \nonumber \\
    & = \mathbb{E}_{(\epsilon_k)}\left[\sup_{f \in \mathcal{F}}  \sum_{k=1}^{M} \epsilon_k \sum_{j=1}^h v_j \left[\langle {u}_j, a_k {x}_k \rangle\right]_+\right] \nonumber \\
    & = \mathbb{E}_{(\epsilon_k)}\left[\sup_{f \in \mathcal{F}} \sum_{j=1}^h v_j \sum_{k=1}^{M} \epsilon_k  \left[\langle {u}_j, a_k {x}_k \rangle\right]_+\right] \nonumber \\
    &\le \mathbb{E}_{(\epsilon_k)}\left[\sup_{f \in \mathcal{F}} \abs{\sum_{j=1}^h v_j \sum_{k=1}^{M} \epsilon_k  \left[\langle {u}_j, a_k {x}_k \rangle\right]_+}\right] \nonumber \\
    &\le \mathbb{E}_{(\epsilon_k)}\left[\sup_{f \in \mathcal{F}} \sum_{j=1}^h \alpha_j \abs{ \sum_{k=1}^{M} \epsilon_k  \left[\langle {u}_j, a_k {x}_k \rangle\right]_+}\right] \nonumber \\
    & \leq  \sum_{j=1}^h \alpha_j \mathbb{E}_{(\epsilon_k)} \sup_{f \in \mathcal{F}} \abs{\sum_{k=1}^{M} \epsilon_k   \left[\langle {u}_j, a_k {x}_k \rangle\right]_+}. 
    \nonumber
    \\
    & =  \sum_{j=1}^h \alpha_j \mathbb{E}_{(\epsilon_k)} \sup_{u_j: \|u_j\| \le \beta_j } \abs{\sum_{k=1}^{M} \epsilon_k   \left[\langle {u}_j, a_k {x}_k \rangle\right]_+}. \label{eqn:nn1}
\end{align}
We bound the expectations in \eqref{eqn:nn1} using Ledoux-Talagrand contraction \cite[Theorem 4.12]{ledoux1991probability}.
\begin{thm}[Ledoux-Talagrand contraction]
Let $F: \R_+ \to \R_+$ be convex and increasing. Further let $\varphi_i$, $i \in [M]$ be 1-Lipschitz functions such that $\varphi(0) = 0$. Then, for any bounded subset $T \subset \R^M$,
\[
\bbE_{(\eps_i)} F \left(\frac12 \sup_{t \in T} \abs{\sum_{i=1}^M \eps_i \varphi_i(t_i)} \right) \le
\bbE_{(\eps_i)} F \left( \sup_{t \in T} \abs{\sum_{i=1}^M \eps_i t_i} \right).
\]
\end{thm}
To apply this result, for each $j$ notice that
\begin{equation}
    \mathbb{E}_{(\epsilon_k)} \sup_{{u}_j: \norm{{u}_j} \leq \beta_j} \abs{\sum_{k=1}^{M} \epsilon_k  \left[\langle {u}_j, a_k  {x}_k \rangle\right]_+} = \mathbb{E}_{(\epsilon_k)} \sup_{{t} \in T_j}{ \abs{\sum_{k=1}^M \epsilon_k \left[t_k \right]_+} } \nonumber
\end{equation}
where ${t} = \left(t_1, t_2,\dots, t_M \right)^{T}$ and
\begin{equation}
     T_j = \left\{{t} = \left(\langle{u}_j , a_1 {x}_1 \rangle, \langle{u}_j , a_2 {x}_2 \rangle, \dots, \langle{u}_j , a_M {x}_M \rangle \right)^{T} \in \mathbb{R}^M : \norm{{u}_j}\leq \beta_j  \right\} \nonumber
\end{equation}
which is clearly bounded. Now taking $F$ to be the identity and $\varphi_i = \left[ \cdot \right]_+$, we have
\begin{align}
    \mathbb{E}_{(\epsilon_k)} \sup_{{u}_j: \norm{{u}_j} \leq \beta_j}\abs{ \sum_{k=1}^{M} \epsilon_k  \left[\langle {u}_j, a_k  {x}_k \rangle\right]_+} 
    & \leq 2\mathbb{E}_{(\epsilon_k)} \sup_{{u}_j: \norm{{u}_j} \leq \beta_j} \abs{\sum_{k=1}^{M} \epsilon_k  \langle {u}_j, a_k  {x}_k \rangle} \nonumber \\
    & = 2\mathbb{E}_{(\epsilon_k)} \sup_{{u}_j: \norm{{u}_j} \leq \beta_j} \abs{ \left\langle {u}_j, \sum_{k=1}^{M} \epsilon_k a_k  {x}_k \right\rangle} \nonumber
    \\
    & = 2\mathbb{E}_{(\epsilon_k)} \left\langle \beta_j \frac{\sum_{k=1}^{M} \epsilon_k a_k  {x}_k}{\| \sum_{k=1}^{M} \epsilon_k a_k  {x}_k \|}, \sum_{k=1}^{M} \epsilon_k a_k  {x}_k \right\rangle \nonumber
    \\    
    & = 2 \beta_j \mathbb{E}_{(\epsilon_k)} 
    \sqrt{\left\| \sum_{k=1}^{M} \epsilon_k a_k  {x}_k \right\|^2 }\label{eqn:csequality}
    \\
    & \le 2 \beta_j \sqrt{ \mathbb{E}_{(\epsilon_k)} 
    \left\| \sum_{k=1}^{M} \epsilon_k a_k  {x}_k \right\|^2 }\label{eqn:nnjensen}
    \\
    & \le 2 \beta_j \sqrt{ \sum_{k=1}^{M} a_k^2  \|{x}_k\|^2 }\label{eqn:nnrad} \\
    & \le 2 \| \sX \|_2 \beta_j \sqrt{ \sum_{k=1}^{M} a_k^2}, \label{eqn:nn2} 
\end{align}
where \eqref{eqn:csequality} uses the condition for equality in Cauchy-Schartz, \eqref{eqn:nnjensen} uses Jensen's inequality, and \eqref{eqn:nnrad} uses independence of the $\eps_k$. The result now follows from \eqref{eqn:nn1} and \eqref{eqn:nn2}.



\subsection{Proof of Theorem \ref{thm:multirad}}

We first review the following properties of the supremum which are easily verified.
\begin{enumerate}
\item[P1]
For any real-valued functions $f_1,f_2: \mathcal{X} \to \mathbb{R}$, 
$$
\sup_{x}f_1(x) - \sup_{x}f_2(x)
\le \sup_{x} (f_1(x)-f_2(x)).
$$
\item[P2]
For any real-valued functions $f_1,f_2: \mathcal{X} \to \mathbb{R}$, 
$$
\sup_{x} (f_1(x)+f_2(x)) \le \sup_{x}f_1(x) + \sup_{x}f_2(x).
$$
\item[P3]
$\sup(\cdot)$ is a convex function, i.e., if $(x_\lambda)_{\lambda \in \Lambda}$ and
$(x'_\lambda)_{\lambda \in \Lambda}$ are two sequences (where $\Lambda$ is possibly
uncountable), then $\forall \alpha
\in [0,1]$,
\begin{align*}
\sup_{\lambda \in \Lambda} (\alpha x_\lambda + (1-\alpha) x'_\lambda) \le \alpha \sup_{\lambda \in \Lambda} x_\lambda + (1-\alpha) \sup_{\lambda \in \Lambda}x'_\lambda.
\end{align*}
\end{enumerate}

Introduce the variable $S$ to denote all realizations $X_{ij}^\sigma$, $1 \in [N], \sigma \in \{-,+\}, j \in [n_i^\sigma]$. We would like to bound
\[
\xi(S) := \sup_{f \in \sF} \abs{
\sum_{i=1}^N w_i \left(\frac12 \sum_{\sig \in \{\pm 1\}} \left[ \frac1{n_i^\sig} \sum_{j=1}^{n_i^\sig} \ell^{\cp}_\sig(f(X_{ij}^\sig))\right] - \sE(f) 
\right)} .
\]
Introduce 
\begin{align*}
\xi^+(S) &:= \sup_{f \in \sF} 
\sum_{i=1}^N w_i \left(\frac12 \sum_{\sig \in \{\pm 1\}} \left[ \frac1{n_i^\sig} \sum_{j=1}^{n_i^\sig} \ell^{\cp}_\sig(f(X_{ij}^\sig))\right] - \sE(f) 
\right), \\
\xi^-(S) &:= \sup_{f \in \sF} 
-\sum_{i=1}^N w_i \left(\frac12 \sum_{\sig \in \{\pm 1\}} \left[ \frac1{n_i^\sig} \sum_{j=1}^{n_i^\sig} \ell^{\cp}_\sig(f(X_{ij}^\sig))\right] - \sE(f) 
\right).
\end{align*}

Assume \IIM holds. Since the realizations $X_{ij}^\sigma$ are independent, we can apply the Azuma-McDiarmid bounded difference inequality \citep{mcdiarmid} to $\xi^+$ and to $\xi^-$. We will show that the same bound on $\xi^+$ and $\xi^-$ holds with probability at least $1 - \delta/2$. Combining these bounds gives the desired bound on $\xi$. We consider $\xi^+$ below, with the analysis for $\xi^-$ being identical.

\begin{defn}
Let $A$ be some set and $\phi:A^n \rightarrow R$.
We say $\phi$ satisfies the bounded difference assumption if $\exists c_1, \ldots, c_n \geqslant 0$ $s.t.$ $\forall i, 1
\leqslant i \leqslant n$
\begin{center}
$\underset{x_1, \ldots, x_n, x_i^\prime \in A}{\sup} |\phi(x_1, \ldots, x_i, \ldots, x_n) - \phi(x_1, \ldots, x_i^{\prime},
\ldots, x_n)| \leqslant c_i$
\end{center}
That is, if we substitute $x_i$ to $x_i^{\prime}$, while keeping other $x_j$ fixed,
$\phi$ changes by at most $c_i$.
\end{defn}

\begin{lemma}[Bounded Difference Inequality]
\label{lem:bdi}
Let $X_1, \ldots, X_n$ be arbitrary independent random variables on set $A$ and
$\phi:A^n \rightarrow R$ satisfy the bounded difference assumption.
Then $\forall t > 0$
\[\Pr \{ \phi(X_1, \ldots, X_n) - \bbE [\phi(X_1, \ldots, X_n)] \geqslant t \}
\leqslant e^{- \frac{2t^2}{\sum_{i=1}^n c_i^2}}.
\]
\end{lemma}

To apply this result to $\xi^+$, first note that for any $f \in \sF, x \in \sX$, and $y \in \{-1,1\}$,
\begin{align*}
|\ell^{\cp_i}(f(x),y)| &\le |\ell^{\cp_i}(0,y)| + |\ell^{\cp_i}(f(x),y) - \ell^{\cp_i}(0,y)| \\
&\le |\ell^{\cp_i}|_0 + |\ell^{\cp_i}||f(x)| \\
&\le |\ell^{\cp_i}|_0 + |\ell^{\cp_i}|A.
\end{align*}
If we modify $S$ by replacing some $X_{ij}^\sig$ with another $X'$, while leaving all other values in $S$ fixed, then (by P1) $\xi^+$ changes by at most $2\frac{w_i (|\ell^{\cp_i}|_0 + |\ell^{\cp_i}|A)}{2n_i^\sig}$, and we obtain that with probability at least $1-\delta/2$ over the draw of
$S_1, \ldots, S_N$,
\begin{align*}
\xi^+-\e{\xi^+} &\leq 2\sqrt{\frac12 \sum_{i=1}^N \frac{w_i^2 (|\ell^{\cp_i}|_0 + |\ell^{\cp_i}|A)^2}{\bar{n}_i} \frac{\log (2/\delta)}{2}}\\
&\le 2(1 + A|\ell|) \sqrt{\frac12 \sum_{i=1}^N  \frac{w_i^2}{\bar{n}_i (1 - \cpm_i - \cpp_i)^2} \frac{\log (2/\delta)}2 },
\end{align*}
where we have used $|\ell^{\cp_i}|_0 \le 1/(1 - \cpm_i - \cpp_i)$ and $|\ell^{\cp_i}| \le |\ell|/(1 - \cpm_i - \cpp_i)$.

To bound $\e{\xi^+}$ we will use ideas from Rademacher complexity theory. Thus let $S'$ denote a separate (ghost) sample of corrupted data $(\ubar{X}_{ij}^\sig) \stackrel{iid}{\sim} \Pt_\sig^{\cp_i}$, $i=1,\ldots, N$, $\sig \in \{\pm\}$, $j=1,\ldots, n_i^\sig$, independent of the realizations in $S$. Let $\widehat{\bbE}_S[f]$ be shorthand for $\sum_i w_i \sum_{\sig \in \{\pm\}} \frac1{2n_i^\sig} \sum_j \ell_\sig^{\cp_i}(f(X_{ij}^\sig)).$ Denote by $(\eps_{ij}^\sig)$ $ i \in [N], \sig \in \{\pm\}, j \in [n_i^\sig]$, iid Rademacher variables (independent from everything else), and let $\bbE_{(\eps_{ij}^\sig)}$ denote the expectation with respect to all of these variables. We have
\begin{align*}
\e{\xi^+} & = \ee{S}{\sup_{f \in \sF} \sum_{i=1}^N w_i \left( \left[ \sum_{\sig \in \{\pm\}} \frac1{2n_i^\sig} \sum_{j=1}^{n_i^\sig} \ell^{\cp_i}_\sig(f(X_{ij}^\sig))\right] - \sE_P^\ell(f) \right)} \\
&= \ee{S}{\sup_{f \in \sF} \Bigg(  \widehat{\bbE}_S[f]
 - \ee{S'}{\widehat{\bbE}_{S'}[f]} \Bigg)} \\
& \qquad \text{(by writing $\sE_P^\ell(f) = \sum w_i \sE_{P^{\cp_i}}^{\ell^{\cp_i}}(f)$ and applying Prop. \ref{prop:unbiased} for each $i$)} \\
&\le \ee{S,S'}{\sup_{f \in \sF} \Bigg( \widehat{\bbE}_S[f]
 - \widehat{\bbE}_{S'}[f]  \Bigg)} \\
& \qquad \text{(by P3 and Jensen's inequality)} \\
&= \ee{S,S'}{\sup_{f \in \sF} \Bigg( \sum_{i=1}^N w_i \sum_{\sig \in \{\pm\}} \frac1{2n_i^\sig} \sum_{j=1}^{n_i^\sig}\ell^{\cp_i}_\sig(f(X_{ij}^\sig)) - \ell^{\cp_i}_\sig(f(\ubar{X}_{ij}^\sig)) \Bigg)} \\
&= \ee{S,S',(\eps_{ij}^\sig)}{\sup_{f \in \sF} \Bigg( \sum_{i=1}^N w_i \sum_{\sig \in \{\pm\}} \frac1{2n_i^\sig} \sum_{j=1}^{n_i^\sig} \eps_{ij}^\sig \Big(\ell^{\cp_i}_\sig(f(X_{ij}^\sig)) - \ell^{\cp_i}_\sig(f(\ubar{X}_{ij}^\sig))\Big) \Bigg)} \\
& \qquad \text{(for all $i, \sig, j$, $X_{ij}^\sig$ and $\ubar{X}_{ij}^\sig$ are iid, and $\eps_{ij}^\sig$ are symmetric)}\\
&\le \ee{S,S',(\eps_{ij}^\sig)}{\sup_{f \in \sF} \sum_{i=1}^N w_i \sum_{\sig \in \{\pm\}} \frac1{2n_i^\sig} \sum_{j=1}^{n_i^\sig} \eps_{ij}^\sig \ell^{\cp_i}_\sig(f(X_{ij}^\sig))} \\
&\qquad \qquad + \ee{S,S',(\eps_{ij}^\sig)}{\sup_{f \in \sF} \sum_{i=1}^N w_i \sum_{\sig \in \{\pm\}} \frac1{2n_i^\sig} \sum_{j=1}^{n_i^\sig} (-\eps_{ij}^\sig) \ell_\sig^{\cp_i}(f(\ubar{X}_{ij}^\sig))}\\
& \qquad \text{(by P2)} \\
&= 2 \bbE_{S}\ee{(\eps_{ij}^\sig)}{\sup_{f \in \sF} \sum_{i=1}^N w_i \sum_{\sig \in \{\pm\}} \frac1{2n_i^\sig} \sum_{j=1}^{n_i^\sig} \eps_{ij}^\sig \ell_\sig^{\cp_i}(f(X_{ij}))}.
\end{align*}

To bound the innermost expectation we use the following result from \citet{meir03jmlr}.
\begin{lemma} 
\label{lemmaZhang}
Suppose $\set{\phi_t}, \set{\psi_t}, t=1,\ldots,T$, 
are two sets of functions on a set $\Theta$ such that for each $t$ and 
$\theta,\theta' \in \Theta, |\phi_t(\theta)-\phi_t(\theta')| \le 
|\psi_t(\theta)-\psi_t(\theta')|$.  Then for all functions $c:\Theta 
\rightarrow \R,$
\begin{equation*}
\bbE_{(\eps_t)} \brac{\sup_\theta \set{c(\theta)+\sum_{t=1}^T \eps_t
\phi_t(\theta)}}
\le \bbE_{(\eps_t)} \brac{\sup_\theta \set{c(\theta)+\sum_{t=1}^T \eps_t
\psi_t(\theta)}}.
\end{equation*}
\end{lemma}

Switching from the single index $t$ to our three indices $i$, $\sig$, and $j$, we apply the lemma with $\Theta = \sF$, $\theta = f$, $c(\theta) = 0$, $\phi_{ij}^\sigma(\theta) = \frac{w_i}{2n_i^\sig} \ell_\sig^{\cp_i}(f(X_{ij}^\sig))$, and $\psi_{ij}^\sig(\theta) = \frac{w_i|\ell|}{2n_i^\sig(1 - \cpm_i - \cpp_i)} f(X_{ij}^\sig)$, where we use $|\ell_\sig^{\cp_i}| \le |\ell|/(1 - \cpm_i - \cpp_i)$.
This yields
\begin{align*}
\e{\xi^+} &\le 2\bbE_{S} \ee{(\eps_{ij}^\sig)}{\sup_{f \in \sF}  \sum_{i=1}^N \frac{w_i |\ell|}{1 - \cpm_i - \cpp_i} \sum_{\sig \in \{\pm\}} \frac1{2n_i^\sig}
\sum_{j=1}^{n_i^\sig} \eps_{ij}^\sig f(X_{ij}^\sig)} \\
&=  2\mathfrak{R}_c^I(\sF),
\end{align*}
To see the second inequality in \eqref{eq:multirad}, by \SR we have
\begin{align*}
    2\mathfrak{R}_c^I(\sF) &\le 2B |\ell|\sqrt{\sum_{i,\sig,j} \left(\frac{w_i}{2n_i^\sig (1 - \cpm_i - \cpp_i)} \right)^2} \\
    &= 2B |\ell|\sqrt{\sum_i \frac{w_i^2}{4 (1 - \cpm_i - \cpp_i)^2} \sum_\sig \frac1{n_i^\sig}} \\
    &= 2B |\ell|\sqrt{\sum_i \frac{w_i^2}{2\bar{n}_i (1 - \cpm_i - \cpp_i)^2}} \\
    &= \sqrt{2}B|\ell| \sqrt{\sum_i \frac{w_i^2}{\bar{n}_i (1 - \cpm_i - \cpp_i)^2}},
\end{align*}
This concludes the proof in the \IIM case.

Now assume \IBM holds. The idea is to apply the bounded difference inequality at the MCM level. If we modify $S$ by replacing $X_{ij}^\sig$ (with $i$ fixed, $j,\sig$ variable) with other values $(X_{ij}^\sig)'$, while leaving all other values in $S$ fixed, then (by P1) $\xi^+$ changes by at most $2w_i (|\ell^{\cp_i}|_0 + |\ell^{\cp_i}|A)$, and we obtain that with probability at least $1-\delta/2$ over the draw of $S$,
\begin{align*}
\xi^+-\e{\xi^+} &\leq \sqrt{\sum_{i=1}^N w_i^2 (|\ell^{\cp_i}|_0 + |\ell^{\cp_i}|A)^2 \frac{\log (2/\delta)}{2}}\\
&\le (1 + A|\ell|) \sqrt{\frac{\log(2/\delta)}2} \sqrt{\sum_{i=1}^N  \frac{w_i^2}{ (1 - \cpm_i - \cpp_i)^2}}.
\end{align*}

To bound $\e{\xi^+}$, we use the same reasoning as in the \IIM case to arrive at
\[
\e{\xi^+} \le 2 \bbE_{S}\ee{(\eps_{i})}{\sup_{f \in \sF} \sum_{i=1}^N w_i \eps_i \sum_{\sig \in \{\pm\}} \frac1{2n_i^\sig} \sum_{j=1}^{n_i^\sig} \ell_\sig^{\cp_i}(f(X_{ij}))},
\]
where now there is a Rademacher variable for every bag. The inner two summations may be expressed
\[
\ee{(\sig,X) \sim \widehat{P}^{\cp_i}}{\ell_\sig^{\cp_i}(f(X))}  
\]
and so by Jensen's inequality and Lemma \ref{lemmaZhang} we have
\begin{align*}
\e{\xi^+} 
&\le 2 \bbE_{S}\bbE_{(\eps_{i})} \left[ \sup_{f \in \sF} \sum_{i=1}^N w_i \ee{(\sig,X) \sim \widehat{P}^{\cp_i}}{\ell_\sig^{\cp_i}(f(X))} \right]  \\
&\le 2 \bbE_{S} \bbE_{{((\sig_i,X_i)\sim \widehat{P}^{\cp_i})}_{i \in [N]}} 
\ee{(\eps_{i})}{\sup_{f \in \sF} \sum_{i=1}^N \eps_i w_i \ell_{\sig_i}^{\cp_i}(f(X_i))} \\
&\le 2 \bbE_{S} \bbE_{{((\sig_i,X_i)\sim \widehat{P}^{\cp_i})}_{i \in [N]}} 
\ee{(\eps_{i})}{\sup_{f \in \sF} \sum_{i=1}^N \eps_i \frac{w_i |\ell|}{1 - \cpm_i - \cpp_i}f(X_i)} \\
&= 2 \mathfrak{R}_{c}^B(\sF)
\end{align*}
This proves the first inequality. To prove the second, by \SR we have
\begin{equation*}
    2\mathfrak{R}_c^B(\sF) \le 2B|\ell| \sqrt{\sum_i \frac{w_i^2}{(1 - \cpm_i - \cpp_i)^2}}.
\end{equation*}
This concludes the proof.


\subsection{Proof of Theorem \ref{thm:llpempbnd}}

We begin by stating a generalization of Chernoff's bound to correlated binary random variables \cite{panconesi97,impagliazzo2010}.
\begin{lemma}
\label{lem:corrchern}
Let $Z_1, \ldots, Z_m$ be binary random variables. Suppose there exists $0 \le \tau \le 1$ such that for all $I \subset [m]$, $\bbP(\prod_{i \in I} Z_i=1) \le \tau^{|I|}$. Then for any $\epsilon \ge 0$, $\bbP(\sum_{i=1}^m Z_i \ge m(\tau + \epsilon)) \le e^{-2m \epsilon^2}$. 
\end{lemma}

We will first prove the theorem for BP. The result for dominating schemes will then follow easily. Thus, assume the $K$-merging scheme is BP. For now assume \CIBM, which is implied by \CIIM. 

Let $\lpphat_{ik}$ be the larger of the two {\em empirical} label proportions within the $k$th pair of small bags within the $i$th pair of big bags, and similarly let $\lpmhat_{ik}$ be the smaller. Also let $\lpp_{ik}$ be the larger of the two {\em true} label proportions within the $k$th pair of small bags within the $i$th pair of big bags, and similarly let $\lpm_{ik}$ be the smaller.

Let $\eps_0 \in (0,\Delta(1-\tau))$ and let $\eps \in (0,\frac{\Delta(1-\tau)-\eps_0}{1+\Delta}]$. For $i \in [M]$, let $K_i$ be the number of original pairs in the $i$th block (the $i$th pair of big bags) for which $|\lpp_{ik} - \lpm_{ik}| \ge \Delta$, $k \in [K]$ and define $\Omega_{\blp,i}$ to be the event that $K_i \ge K(1 - \tau - \eps)$. By Lemma \ref{lem:corrchern} and \LP, we have $\Pr_{\blp}(\Omega_{\blp,i}^c) \le e^{-2K\eps^2}$.

Also define $\Omega_{\bY,i}$ to be the event that $\Lpphat_i - \Lpmhat_i \ge \bbE_{\bY|\blp}[\Lpphat_i - \Lpmhat_i] - \eps = \Lpp_i - \Lpm_i - \eps$. Note that conditioned on $\blp$, $\Lpphat_i - \Lpmhat_i = \frac1{K} \sum_{k=1}^K (\lpphat_{ik} - \lpmhat_{ik})$ is the sum of $K$ independent random variables with range $[0,1]$ (here we use the definition of BP and conditional independence of the small bags under \CIBM). By Hoeffding's inequality, $\bbP_{\bY|\blp}(\Omega_{\bY,i}^c) \le e^{-2K\eps^2}$.

Now define $\Omega_{\blp} := \bigcap_{i=1}^M \Omega_{\blp,i}$ and $\Omega_{\bY} := \bigcap_{i=1}^M \Omega_{\bY,i}$. Also define $\Theta$ to be the event that the first inequality in \eqref{eqn:llpepmbnd} does not hold. Then
\begin{align*}
\bbP(\Theta) 
&\le \bbP(\Theta|\Omega_{\blp} \cap \Omega_{\bY}) + \bbP((\Omega_{\blp} \cap \Omega_{\bY})^c) \\
&\le \bbP(\Theta|\Omega_{\blp} \cap \Omega_{\bY}) + \bbP(\Omega_{\blp}^c) + \bbP(\Omega_{\bY}^c) \\
&\le \bbP(\Theta|\Omega_{\blp} \cap \Omega_{\bY}) + \frac{N}{K} e^{-2K\eps^2} + \bbE_{\blp}\ee{\bY|\blp}{\ind{\Omega_{\bY}^c}} \\
&\le \bbP(\Theta|\Omega_{\blp} \cap \Omega_{\bY}) + \frac{2N}{K} e^{-2K\eps^2} \\
&= \ee{\blp,\bY}{\ee{\bX|\blp,\bY}{\ind{\Theta}| \blp, \bY} | \Omega_{\blp} \cap \Omega_{\bY}\ } + \frac{2N}{K} e^{-2K\eps^2}. 
\end{align*}

We next bound the inner expectation of the last line above, which is the conditional probability of $\Theta$ given fixed values of $(\blp,\bY) \in \Omega_{\blp} \cap \Omega_{\bY}$. We will bound this probability the same  argument as in the proof of Thm. \ref{thm:multirad}. To apply that argument, we first need to confirm two things: Conditioned on $\blp, \bY$, (1) for each $i$, $\Lpphat_i - \Lpmhat_i > 0$, and (2) the empirical error $\sEt(f)$ is an unbiased estimate of $\sE_P^\ell$. The first property is given by the following.

\begin{lemma}
\label{lem:lpdiff}
Conditioned on $(\blp,\bY) \in \Omega_{\blp} \cap \Omega_{\bY}$, for all $i \in [M]$
\[
\Lpphat_i - \Lpmhat_i \ge \Lpp_i - \Lpm_i - \eps \ge \eps_0. 
\]
\end{lemma}
\begin{proof}
Fix $(\blp,\bY) \in \Omega_{\blp} \cap \Omega_{\bY}$. Let $i \in [M]$. By definition of $\Omega_{\bY}$, 
\begin{align*}
\Lpphat_i - \Lpmhat_i &\ge \bbE_{\bY|\blp}[\Lpphat_i - \Lpmhat_i] - \eps \\
&= \left( \frac1{K} \sum_{k=1}^K \bbE_{\bY|\blp}[\lpphat_{ik} - \lpmhat_{ik}] \right) - \eps \\
&\ge \left( \frac1{K} \sum_{k=1}^K \lpp_{ik} - \lpm_{ik} \right) - \eps.
\end{align*}
To see the last step, let $U$ and $V$ be random variables with means $p$ and $q$. Then $\bbE[\max(U,V) - \min(U,V)] = \bbE[|U - V|] \ge |\bbE[U - V]| = |p - q| = \max(p,q) - \min(p,q)$, by Jensen's inequality. Here we have again used the definitions of BP and \CIBM.

By definition of $\Omega_{\blp}$, $\lpp_{ik} - \lpm_{ik} \ge \Delta$ for $K_i \ge K(1-\tau - \eps)$ values of $k \in [K]$. From this we conclude that $\Lpphat_i - \Lpmhat_i \ge \Delta(1-\tau - \eps) - \eps \ge \eps_0$, where the last step follows from $\eps \le \frac{\Delta(1-\tau) - \eps_0}{1 + \Delta}$.
\end{proof}

For the second property, recall $\sEt(f) = \sum_i w_i \sEt_i(f)$ with $w \in \Delta^M$ and $w_i \propto (\Lpphat_i - \Lpmhat_i)^2$. We note that $\ee{\bX|\blp,\bY \in \Omega_{\blp} \cap \Omega_{\bY}}{\sEt_i(f)}$ is well defined because $|\ell^{\cphat_i}(f(x))|$ is bounded for $x \in \sX$. This follows from the assumption $\sup_{f \in \sF, x \in \sX} |f(x)| \le A < \infty$, the fact that $\ell^{\cphat_i}$ is Lipschitz continuous on $\Omega_{\blp} \cap \Omega_{\bY}$ by Lemma \ref{lem:lpdiff}, and the observation $|\ell^{\cphat_i}(f(x))| \le |\ell^{\cphat_i}|_0 + |\ell^{\cphat_i}|A$.
\begin{lemma}
\label{lem:llpunbiased}
For all $f \in \sF$, $\ee{\bX|\blp,\bY \in \Omega_{\blp} \cap \Omega_{\bY}}{\sEt_i(f)} = \sE_{P}^\ell(f)$. 
\end{lemma}
\begin{proof}
Recall that $X_{mj}$ denotes the $j$th instance in the $m$th original (pre-merging) small bag, $m \in [2N]$, $j \in [n]$, and that $Y_{mj}$ denotes the corresponding label. We have
\begin{multline*}
\ee{\bX|\blp,\bY \in \Omega_{\blp} \cap \Omega_{\bY}}{\sEt_i(f)} \\
\begin{aligned}
&= \frac12 \ee{\bX|\blp,\bY \in \Omega_{\blp} \cap \Omega_{\bY}}{
\frac1{nK} \sum_{m \in I_i^+} \sum_{j=1}^n \ell_+^{\cphat_i}(f(X_{mj})) + \frac1{nK} \sum_{m \in I_i^-} \sum_{j=1}^n \ell_-^{\cphat_i}(f(X_{mj})) } \\
&= \frac12 \bbE_{\bX|\blp,\bY \in \Omega_{\blp} \cap \Omega_{\bY}} \Bigg[ \Lpphat_i \frac1{nK\Lpphat_i} \sum_{m \in I_i^+} \sum_{j:Y_{mj} = 1} \ell_+^{\cphat_i}(f(X_{mj})) \\
& \qquad + (1-\Lpphat_i) \frac1{nK(1-\Lpphat_i)} \sum_{m \in I_i^+} \sum_{j:Y_{mj} = -1} \ell_+^{\cphat_i}(f(X_{mj}))  \\
& \qquad + \Lpmhat_i \frac1{nK\Lpmhat_i} \sum_{m \in I_i^-} \sum_{j:Y_{mj} = 1} \ell_-^{\cphat_i}(f(X_{mj})) \\
& \qquad + (1-\Lpmhat_i) \frac1{nK(1-\Lpmhat_i)} \sum_{m \in I_i^-} \sum_{Y_{mj} = -1} \ell_-^{\cphat_i}(f(X_{mj})) \Bigg] \\
&= \frac12 \Big\{ \Lpphat_i \ee{X \sim \Pp}{\ell_+^{\cphat_i}(f(X))} + (1-\Lpphat_i) \ee{X \sim \Pm}{\ell_+^{\cphat_i}(f(X))} \\
& \qquad +  \Lpmhat_i \ee{X \sim \Pp}{\ell_-^{\cphat_i}(f(X))} + (1-\Lpmhat_i) \ee{X \sim \Pm}{\ell_-^{\cphat_i}(f(X))} \Big\} \\
&= \frac12 \Bigg\{ \ee{X \sim \Pp^{\cphat_i}}{\ell_+^{\cphat_i}(f(X))} + \ee{X \sim \Pm^{\cphat_i}}{\ell_-^{\cphat_i}(f(X))} \Bigg\} \\
&= \sE_{P}^\ell(f)
\end{aligned}
\end{multline*}
where the third step uses the definition of \CIBM, and the last step uses Prop. \ref{prop:unbiased} and Lemma \ref{lem:lpdiff}.
\end{proof}


By Lemmas \ref{lem:lpdiff} and Lemma \ref{lem:llpunbiased}, we can apply the argument in the proof of Theorem \ref{thm:multirad}, conditioned on $(\blp,\bY) \in \Omega_{\blp} \cap \Omega_{\bY}$, with the estimator $\sEt$ instead of $\sEhat_w$. The only other changes are that in the application of Lemma \ref{lemmaZhang}, we use the bound 
\[
|\ell^{\cphat_i}| \le \frac{|\ell|}{\Lpphat_i - \Lpmhat_i} \le \frac{|\ell|}{\Lpp_i - \Lpm_i - \eps},
\]
and in the final bounds, we upper bound $(\Lpphat_i - \Lpmhat_i)^{-1}$ by $(\Lpp_i - \Lpm_i - \eps)^{-1}$.


\section{Symmetric Losses}

A loss is said to by {\em symmetric} if there exists a constant $K$ such that for all $t$, $\ell(t,1) + \ell(t,-1) = K$. Examples include the 0-1, sigmoid, and ramp losses. For a symmetric loss, $\ell^\cp$ simplifies to
\[
\ell^\cp(t,y) = \frac1{1 - \cpp - \cpm} \ell(t,y) - \frac{K}{1 - \cpp - \cpm}(\cpm \ind{y=1} + \cpp \ind{y=-1}).
\]
Combined with Proposition \ref{prop:unbiased}, this yields
\[
\sE_{P^\cp}^\ell(f) = (1 - \cpp - \cpm) \sE_P^\ell(f) + K \Big(\frac{\cpp + \cpm}2\Big).
\]
Therefore, the two sides have the same minimizer which implies that the BER is {\em immune} to label noise under a mutual contamination model. That is, training on the contaminated data without modifying the loss still minimizes the clean BER. This result has been previously observed for the 0/1 loss \cite{menon15icml} and general symmetric losses \cite{rooyen15tr,charoenphakdee19icml}. The above argument gives a simple derivation from Prop. \ref{prop:unbiased}.

\section{Convexity}

We say that the loss $\ell$ is {\em convex} if, for each $\sig$, $\ell_\sig(t)$ is a convex function of $t$. 
Let $\ell_\sig''$ denote the second derivative of $\ell$ with respect to its first variable. The condition in \eqref{eqn:second} below was used by \citet{natarajan18jmlr} to prove a convexity result an unbiased loss in the class-conditional noise setting. Here we prove a version for MCMs.
\begin{prop}
\label{prop:conv1}
Suppose $\cpm + \cpp < 1$ and let $\ell$ be a convex, twice differentiable loss satisfying 
\begin{equation}
\label{eqn:second}
\ell_+''(t) = \ell_-''(t).
\end{equation}
If $\cp^\sig < \frac12$ for $\sig \in\{\pm\}$, then $\ell^\cp$ is convex.
\end{prop}
Examples of losses satisfying the second order condition include the logistic, Huber, and squared error losses. The result is proved by simply observing
\begin{align*}
(\ell_\sig^\cp)''(t) &= \ell_+''(t) \frac{1 - 2\cp^{-\sig}}{1 - \cpm - \cpp} \\
&\ge 0.
\end{align*}

The statement about $\sEhat_i(f)$ being convex when $f$ is linear was a holdover from an earlier draft and should be disregarded. In the infinite bag size limit, $\sEhat_i(f)$ converges to $\sE_P^\ell(f)$, which is convex in the output of $f$ provided $\ell$ is convex. Sufficient conditions for the convexity of $\sEhat_i(f)$ or $\sEhat_w(f)$ for small bag sizes is an interesting open question.

\section{{\bf (CIBM')} implies \IBM}

Assume that {\bf (CIBM')} holds. To show \IBM, we need to show that for a fixed bag $i$, and for all $j \in [n_i]$, the marginal distribution of $X_{ij}$, conditioned on the bag, is $\lp_i \Pp + (1 - \lp_i) \Pm$. Thus let $A$ be an arbitrary event. Also let $p_i$ be the joint pmf of $Y_{i1},\ldots, Y_{in_i}$, conditioned on the bag. Without loss of generality let $j=1$.
We have
\begin{align}
    \bbP(X_{i1} \in A) &= \bbE_X[\ind{X_{i1} \in A}] \nonumber \\
    &= \bbE_{Y_{i1}, \ldots, Y_{in_i}} \ee{X_{i1}| Y_{i1}, \ldots, Y_{in_i}}{\ind{X_{i1}\in A}} \nonumber \\
    &= \bbE_{Y_{i1}, \ldots, Y_{in_i}} \bbP_{Y_{i1}}(X_{i1}\in A) \label{eqn:cibm1} \\
    &=\sum_{(y_1,\ldots,y_{n_i})\in \{-1,1\}^{n_i}} \bbP_{y_1}(X_{i1} \in A) p_i(y_1,\ldots, y_{n_i}) \nonumber \\
    &= \Pp(A) \sum_{(y_2,\ldots,y_{n_i})\in \{-1,1\}^{n_i-1}} p_i(1,y_2,\ldots, y_{n_i}) \nonumber \\
    & \qquad + \Pm(A) \sum_{(y_2,\ldots,y_{n_i})\in \{-1,1\}^{n_i-1}} p_i(-1,y_2,\ldots, y_{n_i}) \nonumber \\
    &= \lp_i \Pp(A) + (1-\lp_i) \Pm(A), \label{eqn:cimb2}
\end{align}
where \eqref{eqn:cibm1} and \eqref{eqn:cimb2} use {\bf (CIMB')}.

\section{Optimal Bag Matching}

The bound is minimized by selecting weights
$$
w_i \propto \bar{n}_i (\lp_i^+ - \lp_i^-)^2,
$$
which gives preference to pairs of bags where one bag is mostly +1's (large $\lp_i^+$) and the other is mostly -1's (small $\lp_i^-$). With these weights, the \SR bound is proportional to under \CIIM
\[
\sqrt{\left( \sum_{i=1}^N \bar{n}_i (\lp_i^+ - \lp_i^-)^2 \right)^{-1}}.
\]
Here and below, under {\bf (CIBM')}' substitute $\bar{n}_i \to 1$. 

We can optimize the pairing of bags by further optimizing the bound. Consider the unpaired bags $(B_i, \lp_i)$, $i=1,\ldots, 2N$. Recall that $\bar{n}_i = \HM(n_i^+,n_i^-)$.
We would like to pair each bag to a different bag, forming pairs $(\lp_i^+, \lp_i^-)$, such that 
$$
\sum_{i=1}^N \bar{n}_i (\lp_i^+ - \lp_i^-)^2
$$
is maximized. For each $i < j$, let $u_{ij}$ be a binary variable, with $u_{ij}=1$ indicating that the $i$th and $j$th bags are paired. The optimal pairing of bags is given by the solution to the following integer program:
\begin{align}
\max_{u}  & \ \ \sum_{1 \le i < 2N} \sum_{i < j \le 2N} \HM(n_i,n_j) (\lp_i - \lp_j)^2 u_{ij} \label{eqn:ip} \\
 \text{s.t.} & \ \ u_{ij} \in \{0,1\}, \forall i,j \nonumber \\
& \ \  \sum_{i < j} u_{ij} + \sum_{j < i} u_{ji} = 1, \forall i \nonumber
\end{align}
The equality constraint ensures that every bag is paired with precisely one other distinct bag. This problem is known as the ``maximum weighted (perfect) matching" problem. An exact algorithm to solve it was given by \citet{edmonds1965maximum}, and several approximate algorithms also exist for large scale problems.

When $n_i^\sig = n$ for all $i$ and $\sig$, the solution to this integer program is very simple.
\begin{prop}
If $n_i^\sig = n$ for all $i$ and $\sig$, then the solution to \eqref{eqn:ip} is to match the largest $\lp_i$ with the smallest, the second largest $\lp_i$ with the second smallest, and so on.
\end{prop}
\begin{proof}
Suppose the statement is false. Then there exists an optimal solution, and $i$ and $j$, such that $\lpp_i > \lpp_j$ and $\lpm_i > \lpm_j$. Now consider the matching obtained by swapping the bags associated to $\lpm_i$ and $\lpm_j$. Then the objective function increases by
\[
(\lpp_i - \lpm_j)^2 + (\lpp_j - \lpm_i)^2 - (\lpp_i - \lpm_i)^2 - (\lpp_j - \lpm_j)^2 = 2(\lpp_i - \lpp_j)(\lpm_i - \lpm_j) > 0.
\]
This contradicts the assumed optimality. 
\end{proof}

\section{Merging Schemes that Dominate Blockwise-Pairwise}

Let $\Lppbp_{i}$ and $\Lpmbp_{i}$ denote the quantities $\Lpp_i$ and $\Lpm_i$ when the merging scheme is BP, and let $\Lpp_i$ and $\Lpm_i$ refer to any other merging scheme under consideration. Similarly, let $\Lpphatbp_{i}$ and $\Lpmhatbp_{i}$ denote the quantities $\Lpphat_i$ and $\Lpmhat_i$ when the merging scheme is BP, and let $\Lpphat_i$ and $\Lpmhat_i$ refer to any other merging scheme under consideration.

For a $K$-merging scheme that dominates BP, we still have $\Lpphat_i - \Lpmhat_i \ge \Lppbp_{i} - \Lpmbp_{i} - \eps \ge \eps_0 > 0$ on $\Omega_{\blp} \cap \Omega_{\bY}$ by definition of dominating. Hence the same proof goes through in this case, and we may state the following.

\begin{thm}
\label{thm:llpempbnddom}
Let \LP hold. Let $\eps_0 \in (0,\Delta(1-\tau))$. Let $\ell$ be a Lipschitz loss and let $\sF$ satisfy $\sup_{x \in \sX, f \in \sF} |f(x)| \le A < \infty$. Let $\eps \in (0,\frac{\Delta(1-\tau)-\eps_0}{1+\Delta}]$ and $\delta \in (0,1]$. For any $K$-merging scheme that dominates $BP$, under \CIIM, with probability at least $1 - \delta - 2\frac{N}{K} e^{-2K \epsilon^2}$ with respect to the draw of $\blp, \bY, \bX$,
\[
\Lpphat_i - \Lpmhat_i \ge \Lppbp_{i} - \Lpmbp_{i} - \eps \ge \eps_0
\]
and
\begin{equation}
\label{eqn:llpepmbnddom}
\sup_{f \in \sF} \abs{
\sEt(f) - \sE(f)} \le 2 \mathfrak{R}_c^{I}(\sF) + C \sqrt{\frac{\HM((\Lppbp_{i} - \Lpmbp_{i} - \eps)^{-2})}{(N/K)n}} \stackrel{\SR}{\le} D \sqrt{\frac{\HM((\Lppbp_{i} - \Lpmbp_{i} - \eps)^{-2})}{(N/K)n}},
\end{equation}
where $c_i = w_i |\ell|/(\Lppbp_i - \Lpmbp_i - \eps)$, $C = (1 + A|\ell|) \sqrt{\log(2/\delta)}$, and $D=2 B|\ell| + C$. Under {\bf (CIBM)}, the same bounds hold with the same probability if we substitute $\mathfrak{R}_{c}^I(\sF) \to \mathfrak{R}_{c}^B(\sF)$ and $n \to 1$. 
\end{thm}

We conjecture that it is possible to improve the bound for dominating schemes. Using the current proof technique, this would require proving that
\[
\Lpphat_i - \Lpmhat_i \ge \Lpp_{i} - \Lpm_{i} - \eps
\]
with high probability. For example, with BM, this would require a one-sided tail inequality for how the difference between the average of the larger half and the average of the smaller half of $2K$ independent random variables deviates from its mean. The BP scheme was selected as a reference because it is straightforward to prove such a bound for BP using Hoeffding's inequality.

\section{Consistency}

A discrimination rule $\fhat$ is (weakly) consistent if $\sE_P^{\ell}(\fhat) \to \inf_{f} \sE_{P}^{\ell}(f)$ in probability as $N \to \infty$, where the infimum is over all decision functions.

We first note that if we desire consistency wrt the BER defined with 0-1 loss, it suffices to prove consistency wrt the BER defined with a loss $\ell$ that is ``classification calibrated" \cite{bartlett06}. This is because the BER corresponds to a special case of the usual misclassification risk when the class probabilities are equal. Thus, let $\ell$ be Lipschitz and classification calibrated, such as the logistic loss.

We state our consistency result for the discrimination rule
\[
\fhat \in \argmin_{f \in \sF} J(f) := \sEt(f) + \lambda \|f \|_{\sF_k}^2,
\]
where $\sF_k$ is the reproducing kernel Hilbert space associated to a symmetric, positive definite kernel, and $\lambda > 0$.

\begin{thm}
\label{thm:llpconsist}
Let $\sX$ be compact and let $k$ be a bounded, universal kernel on $\sX$. Let $K \to \infty$ such that $N/K \to \infty$ and $N = O(K^\beta)$ for some $\beta > 0$, as $N \to \infty$. Let $\lambda$ be such that $\lambda \to 0$ and $\lambda (N/K) / \log (N/K) \to \infty$ as $N \to \infty$.  Let \LP and \CIBM hold. Then for any merging scheme that dominates BP, 
\begin{equation}
    \sE(\fhat) \to \inf_{f} \sE_{P}^{\ell}(f)
\end{equation}
in probability as $N \to \infty$.
\end{thm}

\begin{proof}
Let $B$ denote the bound on the kernel. By Proposition \ref{prop:sr} and by Theorem \ref{thm:llpempbnddom} applied to $\sF_{B,R}^k$, for all $\eps_0 \in (0,\Delta(1-\tau))$, $\eps \in (0,\frac{\Delta(1-\tau)-\eps_0}{1+\Delta}]$, and $\delta \in (0,1]$, with probability at least $1 - \delta - \frac{N}{K}e^{- 2 K \eps^2}$,
\[
\sup_{f \in B_k(R)} \left| \sEt(f) - \sE_P^\ell(f) \right| \le \frac{D}{\eps_0}\sqrt{\frac{K}{N}}
\]
where $D =(1 + RB|\ell|) \sqrt{\log(2/\delta)} + 2 RB|\ell|$.

Observe that $J(\fhat) \leq J(0) \le \frac{|\ell|_0}{\eps_0}$. Therefore $\lambda 
\|\fhat\|^2 \leq \frac{|\ell|_0}{\eps_0} - \sEt(\fhat)\leq \frac{2|\ell|_0}{\eps_0}$ and 
so $\|\fhat\|^2 \leq \frac{2|\ell|_0}{\eps_0 \lambda}$.

Set $R = \sqrt{\frac{2|\ell|_0}{\eps_0 \lambda}}$. Note that $R$ grows asymptotically because $\lambda$ shrinks. We just saw that $\fhat \in B_k(R)$. 

Let $\epsilon>0$. Fix $f_\epsilon \in \sF_k$ s.t. $\sE_P^\ell(f_\epsilon) \leq \inf_{f} \sE_{P}^{\ell} + 
\epsilon/2$, possible since $k$ is universal \citep{steinwart08}. Note that $f_\epsilon \in B_k(R)$ for $N$ sufficiently 
large. In this case the generalization error bound implies that with probability $\geq 
1 - \delta - \frac{N}{K}e^{- 2 K \eps^2}$,
\begin{align*}
\sE_P^\ell(\fhat) &\leq \sEt(\fhat) + 
\frac{D}{\eps_0}\sqrt{\frac{K}{N}} \nonumber \\
&\leq \sEt(f_\epsilon) + \lambda\|f_\epsilon\|^2 - 
\lambda\|\fhat\|^2 + \frac{D}{\eps_0} \sqrt{\frac{K}{N}} \nonumber \\
&\leq \sEt(f_\epsilon) + \lambda\|f_\epsilon\|^2 + 
\frac{D}{\eps_0} \sqrt{\frac{K}{N}} \\
&\leq \sE_P^\ell(f_\epsilon) + \lambda\|f_\epsilon\|^2 + 
\frac{2D}{\eps_0} \sqrt{\frac{K}{N}}.
\end{align*}
Taking $\delta = K/N$, the result now follows.
\end{proof}

\section{Experimental Details}

The parameters of InvCal \cite{rueping10} and alter-$\propto$SVM \cite{yu13icml} are tuned by five-fold cross validation. We only consider the RBF kernel. Following \cite{yu13icml}, the parameters for both methods were set as follows. The kernel bandwidth $\gamma$ of the RBF kernel is chosen from $\left\{0.01, 0.1, 1 \right\}$. For InvCal, the parameters are tuned from $C_p \in \left\{0.1, 1, 10 \right\}$, and $\epsilon \in \left\{0, 0.01, 0.1 \right\}$. For alter-$\propto$SVM, the parameters are tuned from $C \in \left\{0.1, 1, 10 \right\}$, and $C_p \in \left\{1, 10, 100 \right\}$.

A Matlab implementation of both InvCal and alter-$\propto$SVM was obtained online.\footnote{https://github.com/felixyu/pSVM} These implementations rely on LIBSVM\footnote{https://www.csie.ntu.edu.tw/~cjlin/libsvm/} and CVX\footnote{http://cvxr.com/cvx/}. We modified the code to preform parameter tuning with cross validation as described above. 
LIBSVM contains its own random number generator that was unfortunately not seeded and hence the results for  alter-$\propto$SVM are not reproducible.

For the MAGIC dataset, InvCal takes roughly 30 minutes on 36 cores to complete the experiments for all bag sizes. For the Adult dataset, InvCal takes roughly 60 minutes on 36 cores. For alter-$\propto$-SVM, the approximated runtime on MAGIC dataset is 70 minutes on 144 cores. On Adult dataset, it is 100 minutes on 144 cores.

All three algorithms require random initialization. \citet{yu13icml} randomly initialize their algorithm ten times and take the result with smallest objective value. This was deemed to be computationally excessive, and hence we only consider one random initialization for each method. This could account for the relatively poor performance of alter-$\propto$SVM compared to past reported performance.

We found that in some cases, the code for alter-$\propto$-SVM wouldn't create a variable 'support\_v', which is used to predict the test label. This resulted from LIBSVM not returning any support vectors. If 'support\_v' did not exist for a given fold, we excluded that fold from the cross-validation error estimate. 

For bag size 8, in the experiments with fixed number of bags, on a handful of occasions there are only two bags in the validation data within a given fold of cross-validation, and both bags have the same label proportion. When this occurs, we cannot compute our criterion, and exclude such folds.

\bibliography{myBib}
\bibliographystyle{plainnat}

\end{document}